\documentclass[journal,final,letterpaper,twocolumn]{IEEEtran}
\usepackage{color}
\usepackage{times}
\usepackage{subfigure}
\usepackage{graphicx}
\usepackage{amssymb}
\usepackage{amsmath}
\usepackage{amsthm}
\usepackage{url}
\usepackage{multirow}

\newtheorem{theorem}{Theorem}

\newtheorem{definition}{Definition}

\newtheorem{lemma}[theorem]{Lemma}

\newtheorem{example}[]{Example}

\title{A Self-Paced Regularization Framework for Multi-Label Learning}

\author{Changsheng Li,
        Fan~Wei,
        Junchi~Yan,
        Weishan~Dong,
        Qingshan~Liu,~\IEEEmembership{Senior Member,~IEEE,}
        Xiaoyu Zhang,
        and Hongyuan Zha
\thanks{C. Li and W. Dong are with IBM Research-China, Beijing 100094, China.
Email: \{lcsheng, dongweis\}$@$cn.ibm.com.}
\thanks{F. Wei is with Department of Mathematics, Stanford University. E-mail: fanwei@stanford.edu.
}
\thanks{J. Yan is with East China Normal University, Shanghai, China.
E-mail: jcyan@sei.ecnu.edu.cn}
\thanks{Q. Liu is with Nanjing University of Information Science and Technology, Nanjing 210014, China.
Email: qsliu@nuist.edu.cn.}
\thanks{X. Zhang is with Institute of Information Engineering, Chinese Academy of Sciences, China.
E-mail: zhangxiaoyu@iie.ac.cn}
\thanks{H. Zha is with Georgia Institute of Technology, Atlanta, USA. \protect\\
E-mail: zha@cc.gatech.edu}
}

\begin{document}

\maketitle

\begin{abstract}
In this paper, we propose a novel multi-label learning framework, called Multi-Label Self-Paced Learning (MLSPL), in an attempt to incorporate the self-paced learning strategy into multi-label learning regime. In light of the benefits of adopting the easy-to-hard strategy proposed by self-paced learning, the devised MLSPL aims to learn multiple labels jointly by gradually including label learning tasks and instances into model training from the easy to the hard. We first introduce a self-paced function as a regularizer in the multi-label learning formulation, so as to simultaneously rank
priorities of the label learning tasks and the instances in each learning iteration. Considering that different multi-label learning scenarios often need different self-paced schemes during optimization, we thus propose a general way to find the desired self-paced functions.
 Experimental results on three benchmark datasets suggest the state-of-the-art performance of our approach.
\end{abstract}

\section{Introduction}

Multi-label learning has attracted much attention in the past decade \cite{huang2012multi,zhang2014review}.
Its goal is to learn a classifier to map the input instance into a label vector space, where each instance is associated with multiple labels instead of one single label. Different from multi-class learning, multiple labels in multi-label learning are often assumed to be correlated with each other.
Often, this correlation among labels is beneficial to accurately predicting labels of test instances.
Due to its empirical success, multi-label learning has been widely applied to various domains including image annotation \cite{nguyen2013multi}, video concept detection \cite{wang2008transductive}, web page categorization \cite{ji2008extracting}, and visual object recognition \cite{bucak2010multi}.

During the past years, many multi-label learning algorithms have been proposed. One simplified approach is to decompose
multi-label learning into multiple independent binary classification problems (one per label or category).
However, such a solution does not consider the relationship among labels, whereas previous studies \cite{elisseeff2001kernel,huang2012multi} have revealed that the label relationship is quite helpful and should be considered.
Therefore, several approaches attempt to exploit label correlations by incorporating external prior knowledge \cite{cai2004hierarchical,rousu2006learning,cesa2006hierarchical,hariharan2010large}.
Considering that the prior knowledge is often unavailable in real applications, many other approaches \cite{elisseeff2001kernel,boutell2004learning,zhu2005multi,yan2007model,qi2007correlative,zhang2010multi,huang2012multi} try to mine label relationships based on training data and incorporate the label correlations into the learning process of multi-label model.
In addition, there are also many works focusing on leveraging other learning techniques for multi-label learning, such as multi-instance multi-label learning \cite{yang2013multi}, active learning for multi-label learning \cite{li2013active}, and multi-label learning combined with multi-kernel learning \cite{ji2009multi}.

The algorithms above treat all the categories equally and also treat all the training instances per category equally when training the model.
However, in real-world scenarios, the complexities of different label learning tasks may differ quite much, and the same for complexities of different training instances in one label learning task.
For example, as shown in Figure \ref{fig:visual}, when learning the label \emph{tiger}, image (b) is clearly harder than image (a), since the color of the tiger in image (b) is quite similar to the background, and the tiger in (b) is partially occlusive by the trees.
Moreover, in image (a), the label \emph{Siberian tiger} is more difficult to learn than the label \emph{tiger}, since Siberian tiger is a subclass of tiger. In addition, many multi-label learning methods are associated with non-convex objective functions, which is prone to local minima especially in the presence of large corruption and bad starting point.
\begin{figure}
\centering
\subfigure[\emph{tiger},\ \emph{snow}, \ \emph{Siberian} \emph{tiger}]{\includegraphics[width=0.46\linewidth]{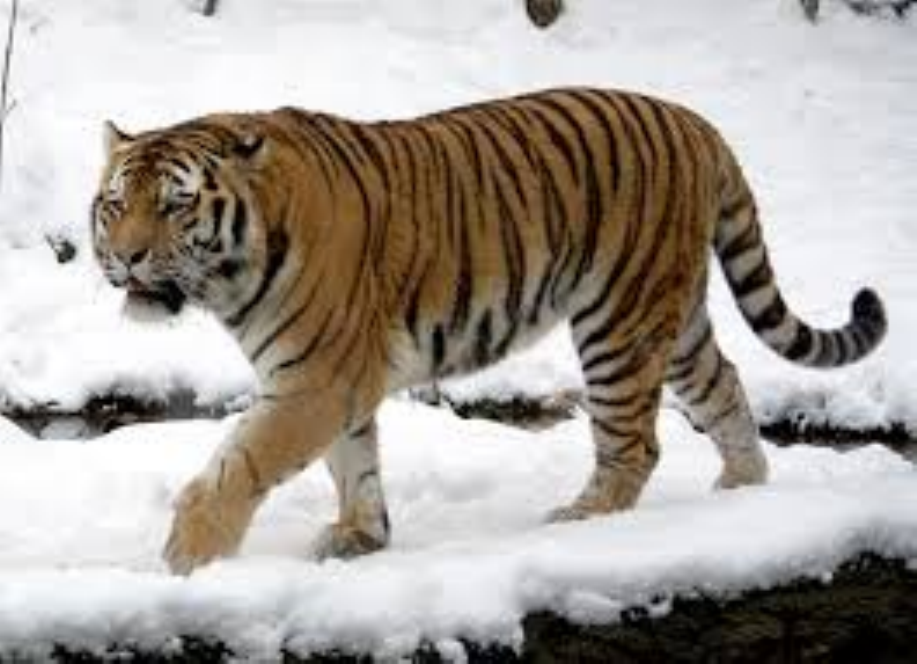}}
\subfigure[\emph{tiger},\ \emph{snow},\ \emph{trees},\ \emph{Siberian} \emph{tiger}]{\includegraphics[width=0.48\linewidth]{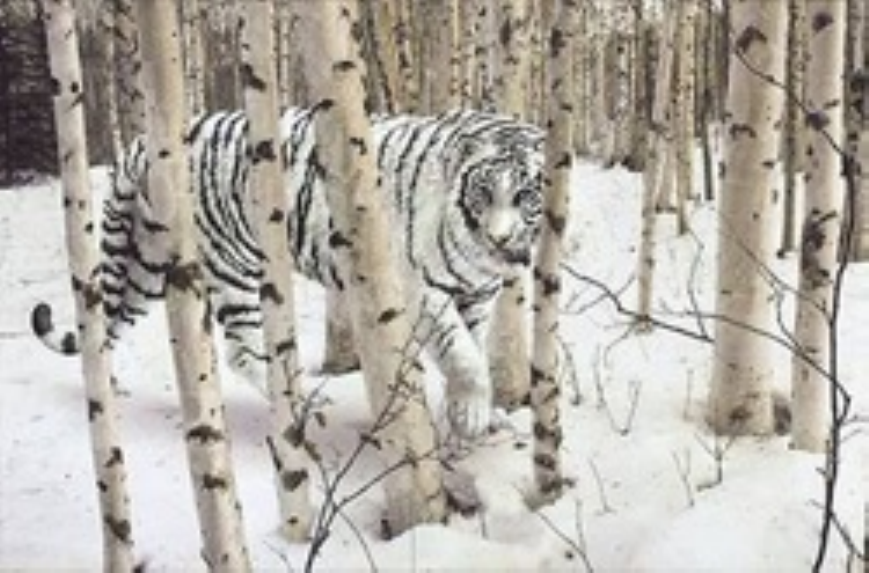}}
\caption{Illustration of complexities of labels and instances.}
\label{fig:visual}
\end{figure}

Inspired by how children learn concepts, self-paced learning \cite{kumar2010self} advocates a paradigm that learning should first consider `simple' or `easy' instances, and then gradually take `complex' or `hard' instances into account. By simulating such a process of human learning, it has been empirically verified that self-paced learning can mitigate the problem of local-minima during iterative learning \cite{kumar2010self,jiang2014self}, and exhibit better generalization behavior in many tasks, such as matrix factorization \cite{zhao2015self}, multi-view learning \cite{xu2015multi}, and multi-instance learning \cite{zhang2015self}. Based on these facts, we conclude that adding the training instances into the learning process in order of complexities can produce a more robust and accurate multi-label learning model.

This paper proposes a novel multi-label learning framework, called  \underline{M}ulti-\underline{L}abel \emph{\underline{S}elf-\underline{P}aced \underline{L}earning} (MLSPL), which is an effort to build a connection between multi-label learning and self-paced learning, in a principled fashion. MLSPL aims to learn multi-label model by introducing a self-paced function as a regularizer that can simultaneously take into consideration the complexities of both instances and labels during learning. Similar to human's learning mechanism, MLSPL should use different learning schemes for different multi-label learning scenarios. To achieve this, we present a general way to find the self-paced functions for the desired learning schemes.
Finally, we tailor a simple yet effective algorithm to solve the optimization problem.
Experimental results on three benchmark datasets demonstrate the effectiveness of the proposed approach, compared to the state-of-the-art methods.
\section{Background}
We first define notations and then briefly introduce the work \cite{huang2012multi} that our approach is originated from.

Let $\mathcal{X}=\mathbb{R}^d$ be the $d$-dimensional input feature space and $\mathcal{Y} = \{-1, +1\}^L$ the finite set of $L$ possible labels. Given a multi-label training set $\mathcal{D} = \{(\mathbf{x}_i , \mathbf{y}_i )\}_{i=1}^n$, where $\mathbf{x}_i =[x_{i1},\ldots, x_{id}] \in \mathcal{X}$ is the $i$-th instance and $\mathbf{y}_i =[y_{i1},\ldots,y_{iL}]\subseteq \mathcal{Y}$ is the label vector associated with $\mathbf{x}_i$. $y_{ij}$ is $+1$ if $\mathbf{x}_i$ has the $j$-th label and $-1$ otherwise.
The goal of multi-label learning is to learn a multi-label learner $h: \mathcal{X} \rightarrow 2^{\mathcal{L}}$
from the training set $\mathcal{D}$, so that we can predict a set of labels for each unseen instance.

As mentioned earlier, most existing multi-label learning methods attempt to exploit the correlations among labels to help learn the classifier $h$. Among these methods, one representative algorithm is  ML-LOC \cite{huang2012multi} which tries to exploit the correlations locally; it assumes that the instances are partitioned into $m$ different clusters and each cluster shares a subset of label correlations. Let $\mathbf{W}=[\mathbf{w}_1,\ldots,\mathbf{w}_L]\in \mathbb{R}^{n \times L}$, $\mathbf{Q}=[\mathbf{q}_1,\ldots,\mathbf{q}_n] \in \mathbb{R}^{m \times n}$, and $\mathbf{A}=[\mathbf{a}_1,\ldots,\mathbf{a}_m] \in \mathbb{R}^{L \times m}$ where $\mathbf{a}_j \in \mathbb{R}^L$ is the mean of all the label vectors in the $j$-th cluster. Before explaining the variables, we first provide its formulation.
\begin{align}\label{obj1}
\min_{\mathbf{W},\mathbf{A},\atop {\mathbf{Q}\in[0,1]^{m\times n}}} \ \sum_{l=1}^L\mathcal{L}(\mathbf{w}_l,\mathbf{Q};\mathcal{D})+ \alpha \Gamma(\mathbf{W}) + \beta \Omega(\mathbf{A}) \qquad\\ =\sum_{l=1}^L\!\mathcal{L}(\mathbf{w}_l,\mathbf{Q};\mathcal{D})\!\!+\!\!\alpha\!\sum_{l=1}^L\!\|\mathbf{w}_l\|^2\!\!+\!\!\beta\!\sum_{i=1}^n\!\sum_{j=1}^mq_{ij}\|\mathbf{y}_i\!\!-\!\!\mathbf{a}_j\|^2 \notag\\
s.t. \ \ \sum_{j=1}^m q_{ij}=1, \forall i\in [1,n], \qquad\qquad\qquad\qquad\nonumber
\end{align}
where $\mathbf{W}$ is the learned weight matrix with each column representing the weight vector for the corresponding task.
$\mathcal{L}(\mathbf{w}_l,\mathbf{Q};\mathcal{D})$ is the empirical loss on the training set $\mathcal{D}$ for the $l$-th label, defined as $\mathcal{L}(\mathbf{w}_l,\mathbf{Q};\mathcal{D})=\sum_{i=1}^n\mathcal{L}(\mathbf{w}_l,\mathbf{q}_i;\mathbf{x}_i,y_{il})$. Let $\mathbf{z}_i = [\phi(\mathbf{x}_i);\mathbf{q}_i]$\footnote{The symbol $[\mathbf{u}; \mathbf{v}]$ means flatten $\mathbf{u, v}$ into one single column vector.}  where $\phi$ is a feature mapping induced by a kernel $\kappa$.
 Then $\mathcal{L}(\mathbf{w}_l,\mathbf{q}_i;\mathbf{x}_i,y_{il}) = \max(0,1-y_{il}\mathbf{w}_l^T \mathbf{z}_i)$.
The idea of ML-LOC is as follows. If the label vector $\mathbf{y}_i$ is close to $\mathbf{a}_j$, then it is more likely that $\mathbf{x}_i$  belongs to the $j$-th cluster. Moreover, minimizing the third term in (\ref{obj1}) will result in a larger $q_{ij}$ when $\mathbf{y}_i$ is closer to $\mathbf{a}_j$. Thus, $\mathbf{Q}$ encodes the similarity between instances and clusters.
The first term aims to take the similarity information as additional features for the instance, and incorporates them into the learning process of the classifier.
By jointly optimizing these terms, the global discrimination fitting and local correlation sensitivity can be realized in a unified framework.
\section{Multi-Label Self-Paced Learning}
Here, we will first present a general multi-label learning formulation with self-paced paradigm. Then we give a principled way to find the self-paced functions for realizing desired self-paced schemes. Last, an efficient algorithm is designed to solve the proposed optimization problem.
\subsection{Proposed Formulation}
As can be seen from (\ref{obj1}), the objective function treats  all the training instances and all the label learning tasks equally.
However, we should prioritize learning the easier instances and the easier labels. Moreover, the non-convexity of problem (\ref{obj1}) renders the issue of bad local minima.
To overcome these shortcomings, one interesting principle is that learning should be first done on the easy instances, and then gradually take the hard.
This coincides with the idea for self-paced learning which is inspired by the way humans learn. Indeed, self-paced learning has empirically demonstrated its usefulness for mitigating bad local minima and achieving better model generalization  \cite{kumar2010self,jiang2015self}. More specifically, here we have `easy' and `hard' labels, as well as `easy' and `hard' instances. With the aim to incorporate the self-paced learning paradigm into the multi-label learning regime, we propose the following objective function by considering the complexities of labels and instances in a unified setting:
\begin{align}\label{obj2}
\min_{{\mathbf{W},\mathbf{A},\atop \mathbf{Q}\in[0,1]^{m\times n} },\atop \mathbf{V}\in[0,1]^{n\times L}} \ \sum_{i=1}^n\sum_{l=1}^Lv_i^{(l)}\mathcal{L}(\mathbf{w}_l,\mathbf{q}_i;\mathbf{x}_i,y_{il})+\alpha \Gamma(\mathbf{W}) \nonumber\\
+ \beta \Omega(\mathbf{A})+ f(\mathbf{V}, \lambda)\qquad\qquad \qquad\\
s.t. \ \ \sum_{j=1}^m q_{ij}=1, \forall i\in [1,n], \qquad\qquad \qquad\nonumber
\end{align}
where $\mathbf{v}^{(l)}=[v_1^{(l)},\ldots,v_n^{(l)}]^T$ consists of the weights of $n$ instances for the $l$-th label. $\mathbf{V}=[\mathbf{v}^{(1)},\ldots,\mathbf{v}^{(L)}]$.
Different from (\ref{obj1}), the first term in (\ref{obj2}) is a weighted loss term on the training data $\mathcal{D}$.
$f(\mathbf{V},\lambda)$ denotes the self-paced function or self-paced regularizer used to determine which label learning tasks and the corresponding instances to be selected during training: we can select `easy' labels and `easy' instances for learning at the beginning of training. As the learning is iteratively proceeded, we can gradually add `hard' labels and `hard' instances into the process.

More importantly, since different problems often need different self-paced learning schemes during training, there is no universal self-paced function for all applications. Although many self-paced regularizers have been proposed for various applications \cite{kumar2010self,zhang2015self,jiang2015self}, there lacks of a general method to derive the self-paced functions. In the following, we provide a general method to find the appropriate self-paced functions.
\subsection{Self-Paced Function}\label{self-paced function}
First, we introduce definition of the self-paced function from the recent work \cite{zhao2015self}:
\begin{definition}\label{def1}
Suppose that ${v}$ is a weight variable, $l$ is the loss, and $\lambda$ is the learning pace parameter.
$f({v}, \lambda)$ is called self-paced function, if
\begin{enumerate}
\item $f({v}, \lambda)$ is convex with respect to ${v}\in [0, 1]$;
\item $v^{\ast}$ is monotonically decreasing with respect to $l$, and it holds that $\lim_{l\rightarrow 0}v^{\ast} \leq1$, $\lim_{l\rightarrow \infty}v^{\ast} = 0$;
\item $v^{\ast}$ is monotonically increasing with respect to $\lambda$, and it holds that $\lim_{\lambda\rightarrow 0}v^{\ast}=0$, $\lim_{\lambda\rightarrow \infty}v^{\ast}\leq1$;
\end{enumerate}
\noindent where $v^{\ast}(l, \lambda)=\arg\min_{v\in[0,1]}vl+f(v,\lambda)$ for fixed $l, \lambda$.
\end{definition}

We can see that $v^*(l, \lambda)$ is an S-shaped function in $l$.
In order to find the self-paced function $f$, we first find a family of S-shaped functions $g_\lambda(l)$ with range in $[0,1]$ such that it is monotonically decreasing with respect to $l$, and that $\lim_{l \to \infty} g_\lambda(l) =0$ as well as $\lim_{l \to 0} g_\lambda(l) \leq1$. We further want $g_\lambda(l)$ to increase with respect to $\lambda$ when keeping $l$ fixed. Let $v^* = g_\lambda(l)$ be the $\arg\min$ in Definition \ref{def1}. Thus Conditions 2 and 3 in Definition \ref{def1} are satisfied.
Let $l= s(\lambda,v)$ be the inverse function of $g_\lambda(l)$. Then we propose the original lemma based on Definition \ref{def1}.
\begin{lemma}
{A smooth function $f(v, \lambda)$ is a self-paced function corresponding to $v^*$ if and only if }
{$\frac{\partial f(v, \lambda)}{\partial v} =-s(\lambda,v)$ and $\frac{\partial s(v, \lambda)}{\partial v} \leq 0$ \quad for $v \in [0,1]$.}
\end{lemma}
\begin{proof}
{Since $v^{\ast} = \arg \min v l + f(v, \lambda)$. We need}
\[ \frac{\partial (vl + f(v, \lambda))}{\partial v}  = 0.\]
{By plugging in $v^*$, we have}
\[ l + \frac{\partial f(v^*, \lambda)}{\partial v} = 0.\]
By our definition above, we know that $l = s(\lambda, v)$. So we need to have that $\frac{\partial f(v, \lambda)}{\partial v}  = -s(\lambda, v)$.
In this way we have related $f(v, \lambda)$ with $v^*$ which satisfies Conditions 2 and 3.

Now we {also need Condition 1 to be satisfied. It is equivalent to say that} $\frac{\partial^2 f(v, \lambda)}{\partial v^2}\geq 0.$ Thus we have
\[ \frac{\partial^2 f(v, \lambda)}{\partial v^2} = \frac{\partial ( -s(\lambda, v))}{\partial v} \geq 0.\]
\end{proof}

Since different multi-label learning scenarios often need different self-paced schemes, it is necessary to develop more schemes for exploring this interesting direction. Next, we discuss some examples of self-paced learning schemes.
\begin{example}
{We choose $v^*$ as the arctan function, which is a classical S-shaped  activation function.}
\begin{align}\label{v1} v^*(l, \lambda) = \frac{-\text{arctan} (l - \lambda)+\pi/2}{\pi}.\end{align}
{This function is centrally symmetric around the axis $l = \lambda$. And $v^*$ is invairant under $l - \lambda$.}

In order to obtain the inverse function of $v^*$, we have that ${\cot(\pi v^*)} + \lambda = l$. Therefore $s(\lambda,v) = \lambda + \cot(\pi v)$.
{Thus
\[-s(\lambda,v) = \frac{\partial f(v, \lambda)}{\partial v} = -\lambda - \cot (\pi v).\]}
{Integrating, we can obtain $f(v, \lambda)$ by}
\begin{align}\label{f1}
 f(v, \lambda) =- \int_v s(v, \lambda)= - \lambda v - \ln | \sin (\pi v)| /\pi. \end{align}
We also need to check that $f$ is convex. Therefore it suffices to check that
{$\frac{\partial s(v, \lambda)}{\partial v}  = -\pi^2 / \sin^2(\pi v) \leq 0. $}
Thus we have checked that {$f(v, \lambda)$ is a self-paced function corresponding to $v^*(l, \lambda)$}.
\end{example}

\begin{example}
$v^*$ is the classical sigmoid function. We want
\begin{align}\label{v2} v^*(l, \lambda) = \frac{2}{1+ e^{l / \lambda}}. \end{align}
This $v^*$ is invariant under $l / \lambda$

Clearly $v^*(l, \lambda) \in [0,1]$ as $l \geq 0$.
By solving $l$ in terms of $v^*$ and $\lambda$, we have $l = s(v, \lambda) = \lambda \ln (2/v - 1) .$ Thus we have
\begin{align}\label{f2} f(v, \lambda) = - \int_v s(v, \lambda) =\lambda( (2-v)\ln (2-v) + v \ln v).\end{align}
To check that $f$ is convex in $v$, we just need $\frac{\partial s}{\partial v} \leq 0$. This is the case since
$\frac{\partial s}{\partial v} = \frac{2\lambda}{(v-2)v} \leq 0$ when $v \in [0,1].$
\end{example}

\begin{example}
$v^*$ is the classical tanh function, another well-known activation function.
\begin{align}\label{v3} v^*(l, \lambda) = \frac{1}{1+ e^{2(l - \lambda)}}. \end{align}
This $v^*$ is invariant under $l - \lambda$.
Therefore $l = s(v, \lambda) =\frac{1}{2} \ln (1/v - 1) + \lambda.$
By $f(v, \lambda) = - \int_v s(v, \lambda) $, we have
\begin{align}\label{f3} f(v, \lambda) = \frac{1}{2} ( (1-v)\ln(1- v) + v \ln v) - \lambda v .\end{align}
To check that $f$ is convex in $v$, we just need $\frac{\partial s}{\partial v} \leq 0$. This is the case since
$\frac{\partial s}{\partial v} = \frac{1}{2(v-1)v} \leq 0$ when $v \in [0,1].$
\end{example}

\begin{example}
$v^*$ is another well-known activation function, the exponential function, defined by
\begin{equation} v^*(l, \lambda) =  e^{-l/\lambda}. \label{v4} \end{equation}
Clearly this is an S-shaped curve in $l$. When $l \to 0$ we have $v^*(l, \lambda) \to 1$. Also as $l \to \infty$, we have  $v^*(l, \lambda) \to 0$.
$v^*$ is invariant under $l/\lambda$.
Similar to before, we have $l = s(v, \lambda) =\lambda \sqrt{-\ln v}.$
Thus by $f(v, \lambda) = - \int_v s(v, \lambda)$ we have
\begin{align}\label{f4} f(v, \lambda)
= \lambda \left(-2 v -\frac{\sqrt{\pi} {\text{Erf}(-\sqrt{-\ln{v}})}}{\sqrt{-\ln v}}\right)\sqrt{-\ln v}/2,\end{align}
where $\text{Erf}(x) $ is the error function.
To check that $f$ is convex in $v$, we just need $\frac{\partial s}{\partial v} \leq 0$. This is the case since
$\frac{\partial s}{\partial v} = -\frac{1}{2v\sqrt{-\ln v}}$ which is no greater than 0 when $v \in [0,1].$
\end{example}

In general, after fixing one S-shaped function in terms of $l$, we can create a family of S-shaped functions $g_\lambda(l)$ by deciding whether it is invariant under $l/\lambda^c$ for some positive constant $c$ or $l-\lambda$ as shown in the examples above.
By integrating the corresponding inverse function $s(v, \lambda)$ and then {integrating $s$ against $v$}, we can find the self-paced function $f(v, \lambda)$.

\subsection{Optimization}
\begin{table}
\begin{center}
\begin{tabular}{l}
\hline
\label{spmll}
\textbf{Algorithm 1}   Multi-Label Self-Paced Learning Algorithm \\
\hline
\textbf{Input:} Data matrix ${\mathcal{D}}$, Number of groups $m$;\\
 \ \ \ \ \ \ \ \ \ \ \ \ Regularization parameters $\alpha$ and $\beta$;\\
 \ \ \ \ \ \ \ \ \ \ \ \ Self-paced parameters: $\lambda$, $\mu>1$;\\
1. {Initialize} $\mathbf{Q}$ and $\mathbf{A}$ by K-means, and initialize $\mathbf{W}$ by\\
\ \ \ \ \  bi-class SVM;\\
2. \textbf{while} \emph{not converge}  \textbf{do} \\
3.\ \ \ \ \ \ \  Update $\mathbf{V}$ by (\ref{v1}), (\ref{v2}), (\ref{v3}) or (\ref{v4});\\
4.\ \ \ \ \ \ \  Update $\mathbf{W}$ by solving (\ref{obj5});\\
5.\ \ \ \ \ \ \  Update $\mathbf{Q}$ by solving (\ref{obj6});\\
6.\ \ \ \ \ \ \  Update $\mathbf{A}$ by (\ref{obj8});\\
7.\ \ \ \ \ \ \  $\lambda\leftarrow \lambda \mu$; \% \emph{update the learning pace}\\
8. \textbf{end while} \\
\textbf{Output:} $\mathbf{W}.$\\
\hline
\end{tabular}
\end{center}
\end{table}

We adopt an alternating strategy to solve the optimization problem (\ref{obj2}). We first rewrite the objective function (\ref{obj2}) as:
\begin{align}\label{obj3}
\min_{{\mathbf{W},\mathbf{A},\atop \mathbf{Q}\in[0,1]^{m\times n} },\atop \mathbf{V}\in[0,1]^{n\times L}} \sum_{i=1}^n\sum_{l=1}^Lv_i^{(l)}\mathcal{L}(\mathbf{w}_l,\mathbf{q}_i;\mathbf{x}_i,y_{il})\!\!+\!\!\alpha\!\sum_{l=1}^L\!\|\mathbf{w}_l\|^2\nonumber\\
+\beta\sum_{i=1}^n\sum_{j=1}^mq_{ij}\|\mathbf{y}_i\!\!-\!\!\mathbf{a}_j\|^2 + \sum_{i=1}^n\sum_{l=1}^Lf(v_i^{(l)}, \lambda)\\
s.t. \ \ \sum_{j=1}^m q_{ij}=1, \forall i\in [1,n], \qquad\qquad\qquad\qquad\nonumber
\end{align}
where $f(v_i^{(l)},\lambda)$ can be any function of (\ref{f1}), (\ref{f2}), (\ref{f3}) and (\ref{f4}) as the self-paced regularizer.

\noindent\textbf{i) Solving $\mathbf{V}$ with other variables fixed}: the optimization function (\ref{obj3}) can be decomposed in $n\times L$ individual problems for $v_i^{(l)}$ as:
\begin{align}\label{obj4}
\min_{v_i^{(l)}\in[0,1]} v_i^{(l)}\mathcal{L}(\mathbf{w}_l,\mathbf{q}_i;\mathbf{x}_i,y_{il})+ f(v_i^{(l)}, \lambda)
\end{align}

According to the discussion in Sec. \ref{self-paced function}, the optimal solution $v_i^{(l)}$ can be written in a closed form as (\ref{v1}), (\ref{v2}), (\ref{v3}), (\ref{v4}).

\noindent\textbf{ii) Solving $\mathbf{W}$ with other variables fixed}: problem (\ref{obj3}) can be decomposed into $L$ individual problems for $\mathbf{w}_l$ as:
\begin{align}\label{obj5}
\min_{\mathbf{w}_l} \sum_{i=1}^nv_i^{(l)}\mathcal{L}(\mathbf{w}_l,\mathbf{q}_i;\mathbf{x}_i,y_{il})\!\!+\alpha\|\mathbf{w}_l\|^2
\end{align}
This is a cost-sensitive SVM model, which can be solved by LIBSVM \cite{chang2011libsvm} software package.

\noindent\textbf{iii) Solving $\mathbf{Q}$ with other variables fixed}: the optimization function (\ref{obj3}) can be decomposed in $n$ individual problems for $\mathbf{q}_i$ as:
\begin{align}\label{obj6}
\min_{\mathbf{q}_i} \sum_{l=1}^Lv_i^{(l)}\mathcal{L}(\mathbf{w}_l,\mathbf{q}_i;\mathbf{x}_i,y_{il})
+\beta\sum_{j=1}^mq_{ij}\|\mathbf{y}_i\!\!-\!\!\mathbf{a}_j\|^2 \\
s.t. \ \ \sum_{j=1}^m q_{ij}=1, \mathbf{q}_i\in[0,1]^{m}\qquad\qquad\nonumber
\end{align}

This is a linear programming program, which can be solved efficiently.

\noindent\textbf{iv) Solving $\mathbf{A}$ with other variables fixed}: To obtain $\mathbf{A}$, we can optimize the following objective function:
\begin{align}\label{obj7}
\min_{\mathbf{a}_j} \sum_{i=1}^n q_{ij}\|\mathbf{y}_i-\mathbf{a}_j\|^2
\end{align}

Taking the derivative of (\ref{obj7}) with respect to $\mathbf{a}_j$, and setting it to zero, we have
\begin{align}\label{obj8}
\mathbf{a}_j=\sum_{i=1}^nq_{ij}\mathbf{y}_i
\end{align}

We repeat the above process until the algorithm converges. Algorithm 1 summarizes the algorithm of  multi-label self-paced learning. For testing, we adopt the same strategy as that of ML-LOC, i.e., first predict the code $\mathbf{q}_i$ for unseen test data $\mathbf{x}_i$, and then predict its labels $\mathbf{y}_i$ based on $\mathbf{x}_i$ and $\mathbf{q}_i$.

\begin{table*}
\begin{center}
\caption{Results (mean $\pm$ std.) on the \emph{flags} dataset. Boldface in the table denotes the best performance.}
\label{flags}
\centering
\begin{tabular}{|c|c|c|c|c|c|c|c|}
\hline
 Dataset &   criteria          &   BSVM    &   ML-kNN   &  TRAM & RankSVM & ML-LOC & MLSPL \\
\hline
\multirow{10}{*}{\emph{flags}}  & \multirow{2}{*}{hamming loss $\downarrow$} &0.2966     & 0.3202   &  0.2907 &  0.3009 &   0.2719   & \textbf{0.2518 }\\
                                                        &     &$\pm$0.0181  &  $\pm$0.0196  &     $\pm$0.0111  &  $\pm$0.0222   &$\pm$0.0102 &$\pm$\textbf{0.0067}\\
 \cline{2-8} & \multirow{2}{*}{ranking loss $\downarrow$}       &0.2618   &  0.2465  &     0.2228  &  0.2300   & 0.2415  & \textbf{0.2040}\\
                                                                 &     &$\pm$0.0297  & $\pm$0.0190  &     $\pm$0.0101  & $\pm$0.0098   &$\pm$0.0120 &$\pm$\textbf{0.0011}\\
 \cline{2-8} & \multirow{2}{*}{one error $\downarrow$}    &0.2687   & 0.2446  &    0.1995   &   0.2163  &  0.2613   & \textbf{0.1481} \\
 &     &$\pm$0.0424  &  $\pm$0.0378  &   $\pm$0.0374  & $\pm$0.0148   &$\pm$0.0268 &$\pm$\textbf{0.0022}\\
  \cline{2-8} & \multirow{2}{*}{coverage $\downarrow$}           &4.0433  &   3.9928  &   3.8897  &  3.9165 &  3.9000  & \textbf{3.7320} \\
  &     &$\pm$0.1786  &  $\pm$0.1639  &    $\pm$0.1050  &  $\pm$0.1378  &$\pm$0.1045 &$\pm$\textbf{0.0248}\\
   \cline{2-8} & \multirow{2}{*}{average precision $\uparrow$}    &0.7802  &  0.7891  &   0.8102  &   0.8024 &  0.7939   & \textbf{0.8239} \\
   &     &$\pm$0.0220  &   $\pm$0.0104  &   $\pm$0.0116  &  $\pm$0.0040   &$\pm$0.0086 &$\pm$\textbf{0.0020}\\
\hline
\end{tabular}
\end{center}
\end{table*}

\begin{table*}
\begin{center}
\caption{Results (mean $\pm$ std.) on the \emph{scene} dataset. Boldface in the table denotes the best performance.}
\label{scene}
\centering
\begin{tabular}{|c|c|c|c|c|c|c|c|}
\hline
 Dataset &   criteria          &   BSVM    &   ML-kNN   &  TRAM & RankSVM & ML-LOC & MLSPL \\
\hline
\multirow{10}{*}{\emph{scene}}  & \multirow{2}{*}{hamming loss $\downarrow$}   &0.1183  &  0.1208  &  0.1096 &  0.2567  & 0.1009    & \textbf{0.0946} \\
                                  &     &$\pm$0.0040  &  $\pm$0.0051  &   $\pm$0.0030  &  $\pm$0.0105  & $\pm$0.0036 & $\pm$\textbf{0.0029}\\
 \cline{2-8} & \multirow{2}{*}{ranking loss $\downarrow$}    &0.1099   &  0.1153 &    0.0965  &  0.4492   & 0.0943  & \textbf{0.0837}\\
 &     &$\pm$0.0058  &  $\pm$0.0053  &    $\pm$0.0039 &  $\pm$0.0204   & $\pm$0.0043 & $\pm$\textbf{0.0005}\\
 \cline{2-8} & \multirow{2}{*}{one error $\downarrow$}   &0.2984   &  0.3175  &   0.2921 &  0.7855 &  0.2638  & \textbf{0.2412} \\
 &     &$\pm$0.0161  &  $\pm$0.0124  &  $\pm$0.0089 &  $\pm$0.0309  & $\pm$0.0132 & $\pm$\textbf{0.0040}\\
  \cline{2-8} & \multirow{2}{*}{coverage $\downarrow$}           &0.6395  &  0.6635  &   0.5680  &  2.3308&  0.5610  & \textbf{0.5034} \\
  &     &$\pm$0.0298  &  $\pm$0.0274  &   $\pm$0.0205  &  $\pm$0.0979   & $\pm$0.0227 &$\pm$\textbf{0.0050}\\
   \cline{2-8} & \multirow{2}{*}{average precision $\uparrow$}           &0.8180  &  0.8084  &   0.8279  &   0.4507 &  0.8407  & \textbf{0.8551} \\
   &     &$\pm$0.0097  &  $\pm$0.0071  &  $\pm$0.0055 &  $\pm$0.0205   & $\pm$0.0071 & $\pm$\textbf{0.0024}\\
\hline
\end{tabular}
\end{center}
\end{table*}

\section{Experiments}
To verify the effectiveness of the proposed MLSPL, we perform our method on three benchmark datasets: the \emph{flags} dataset, the \emph{scene} dataset, and the \emph{emotions} dataset.\footnote{These datasets can be downloaded from {http://mulan.sourceforge.net/datasets-mlc.html}} \emph{flags} and \emph{scene} are two image datasets, and \emph{emotions} is a music dataset.
They have 194, 593, and 2407 instances and 7, 6, 6 possible labels, respectively.
To further evaluate MLSPL's performance, we compare it with several state-of-the-art multi-label learning algorithms\footnote{The codes of the compared methods are obtained from the corresponding authors.}.
We first compare with ML-LOC \cite{huang2012multi} that is the most related multi-label learning approach to ours. We also compare with ML-kNN \cite{zhang2007ml} and  RankSVM \cite{elisseeff2001kernel} that  consider first-order and second order correlations, respectively. In addition, we compare with TRAM \cite{kong2013transductive} that is proposed recently.  Finally, we compare with another baseline BSVM \cite{boutell2004learning} that learns a binary SVM for each label.
LibSVM \cite{chang2011libsvm} is used to implement the SVM models for BSVM, ML-LOC and MLSPL.
For the compared methods, the parameters recommended in the corresponding literatures are used. In our method, the regularization parameters $\alpha$ and $\beta$ are set the same with ML-LOC. The initial self-paced parameter $\lambda$ and $\mu$ is searched from $\{10^{-5},10^{-4},10^{-3},10^{-2}\}$ and $\{1.1, 1.2, 1.3, 1. 4, 1.5\}$, and then $\lambda$ is iteratively increased to make  `harder' label tasks and instances included gradually.
We evaluate the performances of the compared approaches with five commonly used multi-label criteria: hamming loss, ranking loss, one error, coverage, and average precision. These criteria measure the performance from different aspects and the detailed definitions can be found in \cite{schapire2000boostexter,zhou2012multi}.
In the following experiments, on each data set, we randomly select $30\%$ instances as the training data, and use the rest $70\%$ instances as the testing data.
We repeat each experiment 10 times, and report the average results as well as standard deviations over the 10 repetitions.

\subsection{General Performance}
In this section, we test the general performance of our method on the three datasets. We use the sigmod activation function as our self-paced learning scheme in this experiment.
Table \ref{flags}, \ref{scene}, and \ref{emotion} summarize the performances of different methods in terms of the five evaluation criteria.
Notice that for average precision, a larger value means a better performance, whereas for the other four criteria, the smaller, the better.
From these tables, we can see that our method outperforms the other approaches significantly on all the three datasets.
For example, on the \emph{flags} dataset, our method achieves 13.4\%, 8.4\%, 25.8\%, 4.1\%, and 1.7\% relative improvement in terms of the five evaluation criteria over TRAM that obtains the second best results on this dataset.
In addition, as can be seen from the objective functions (\ref{obj1}) and (\ref{obj2}), ML-LOC is a special case to our method: when all the entries in $\mathbf{V}$ are 1, our method is reduced to ML-LOC. This shows that our method can improve the prediction performance of the model by jointly considering the complexities of the labels and the instances.
\begin{figure*}
\centering
\subfigure{\includegraphics[width=0.19\linewidth]{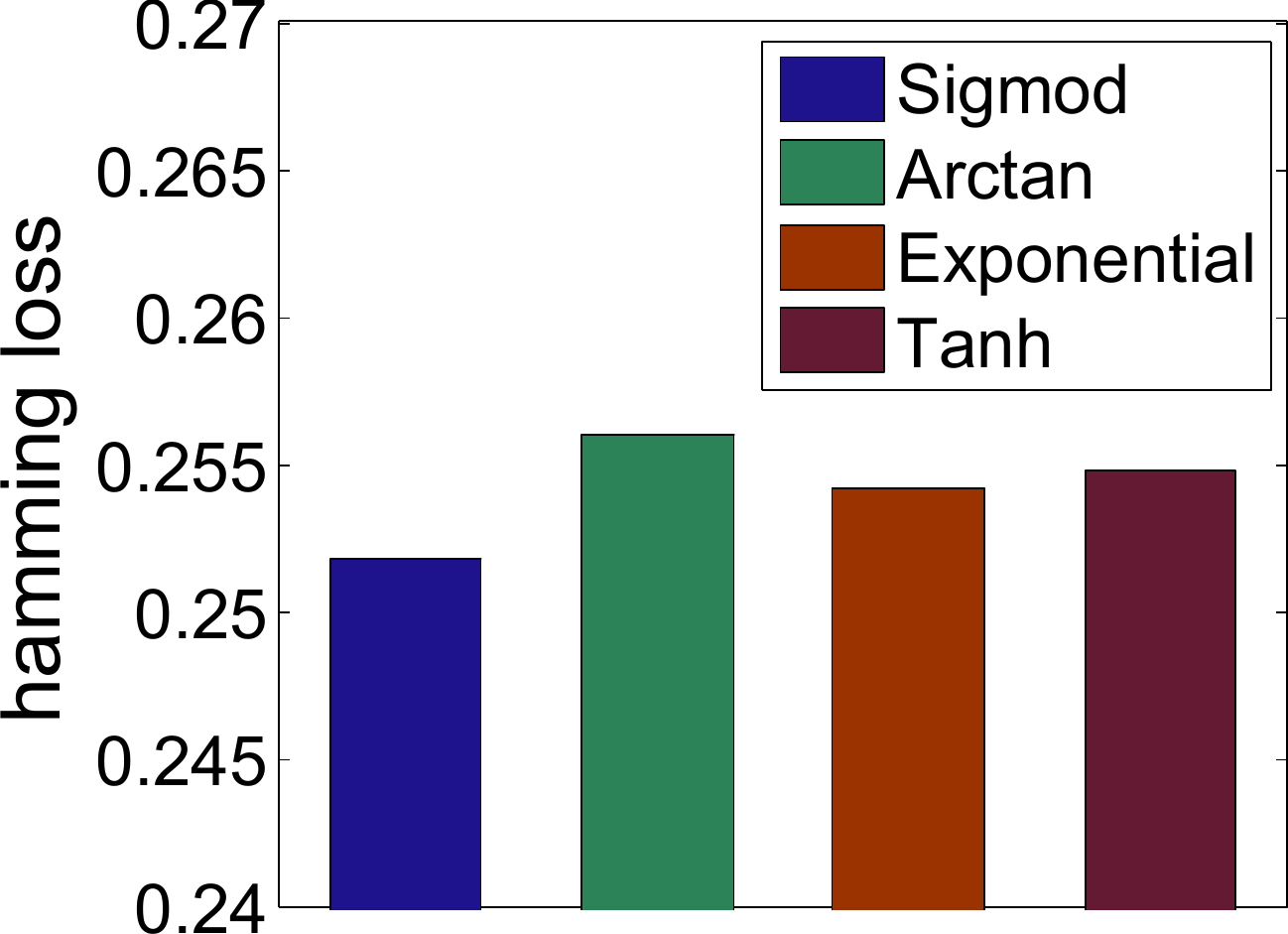}}
\subfigure{\includegraphics[width=0.19\linewidth]{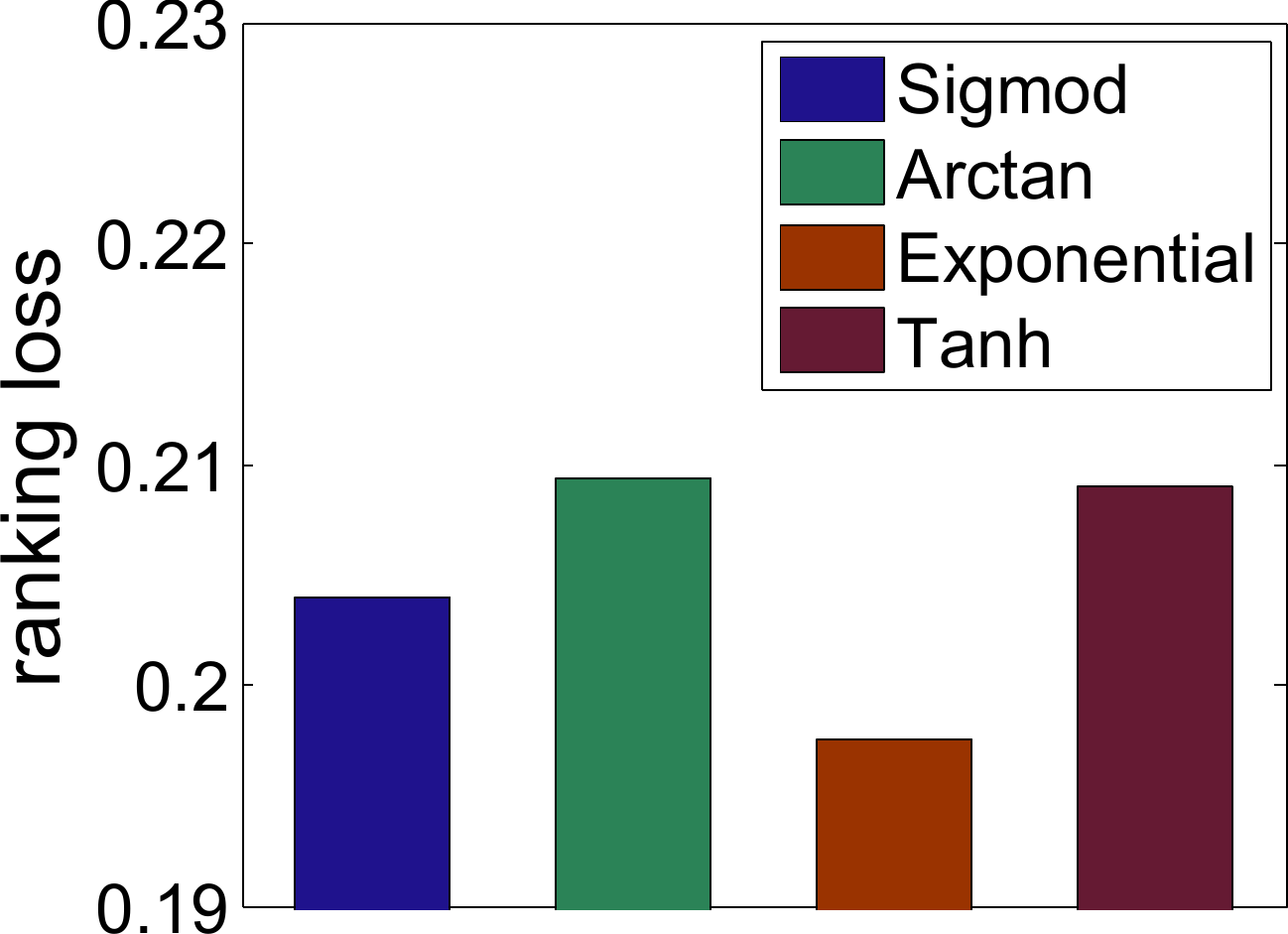}}
\subfigure{\includegraphics[width=0.19\linewidth]{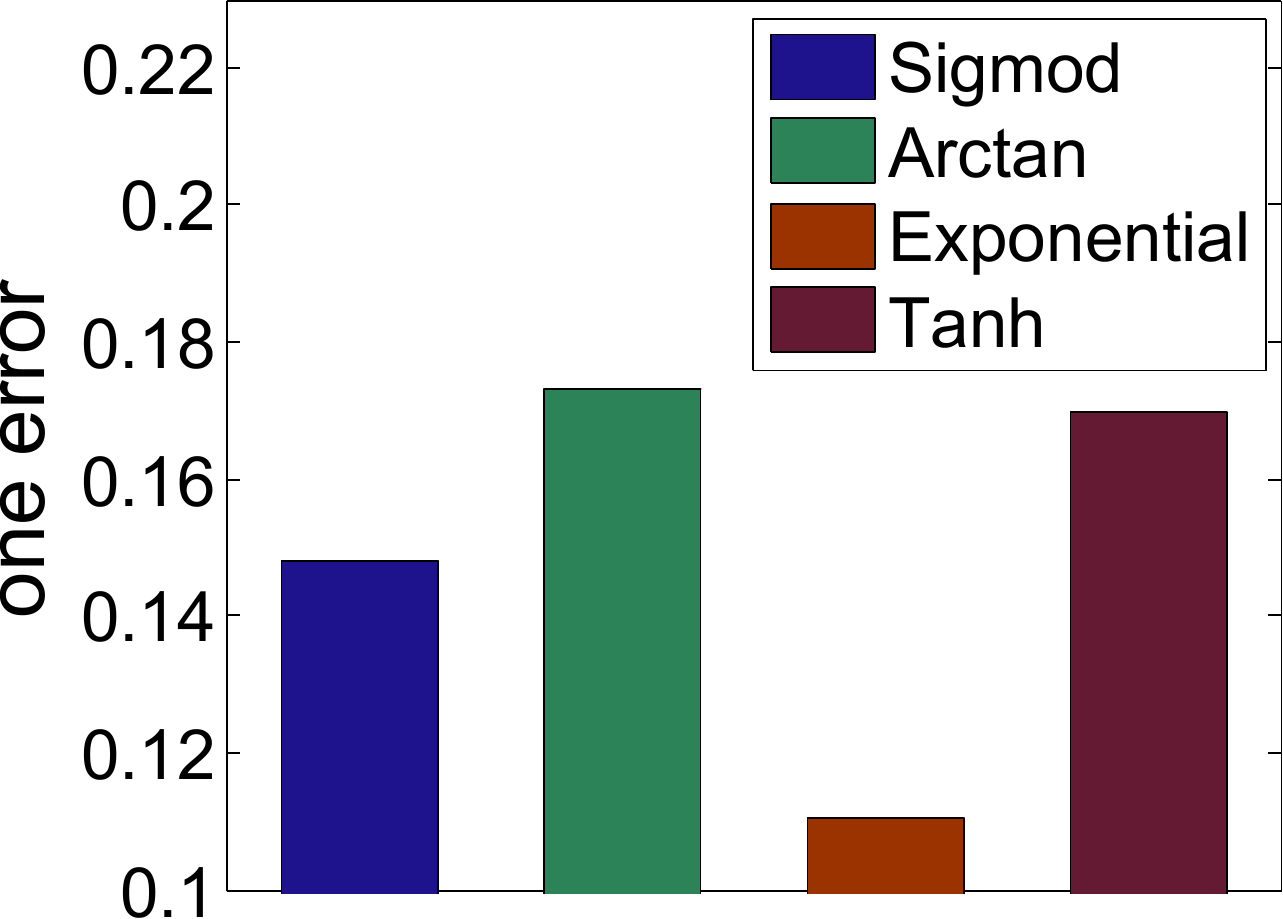}}
\subfigure{\includegraphics[width=0.19\linewidth]{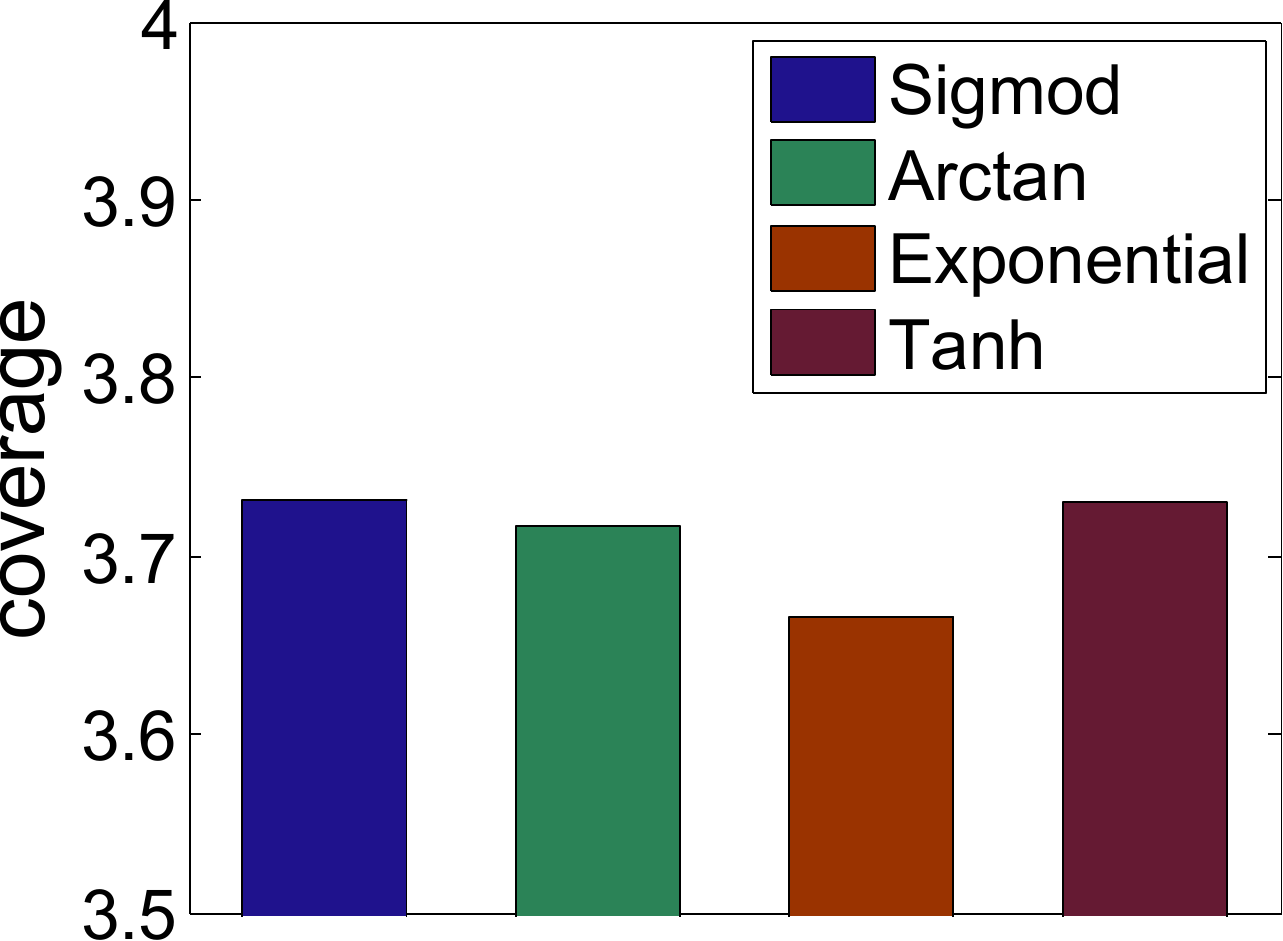}}
\subfigure{\includegraphics[width=0.19\linewidth]{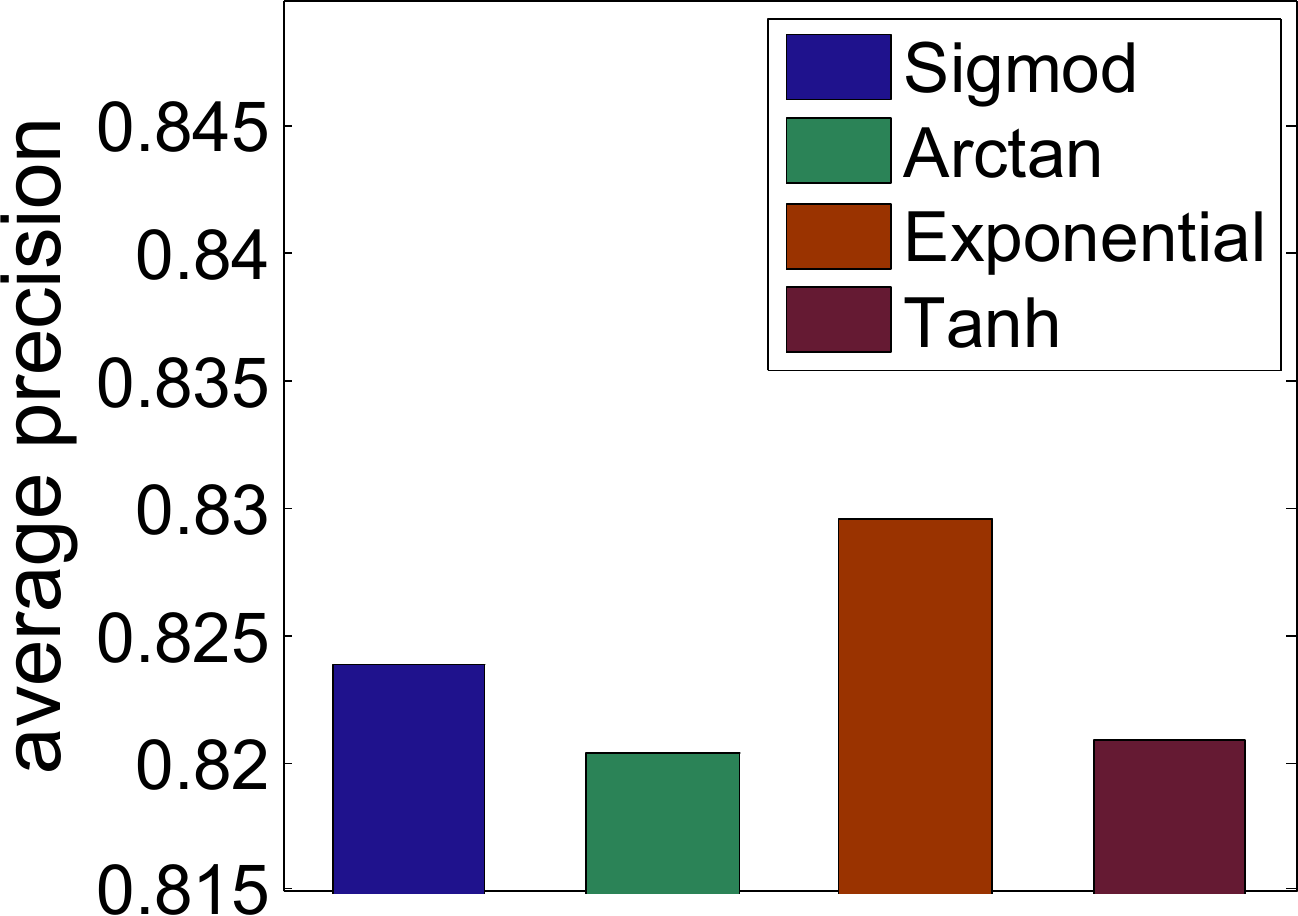}}
\subfigure{\includegraphics[width=0.19\linewidth]{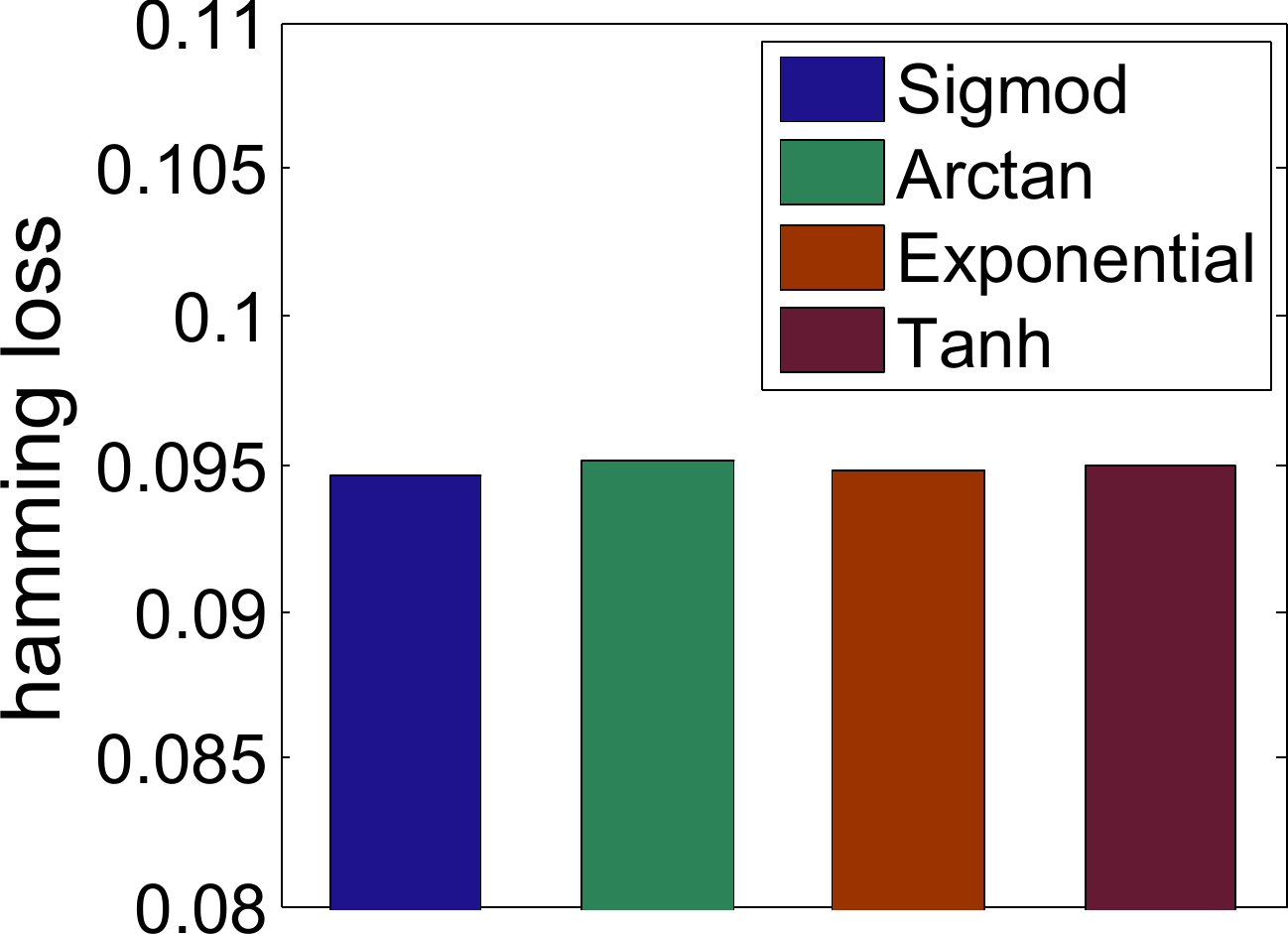}}
\subfigure{\includegraphics[width=0.19\linewidth]{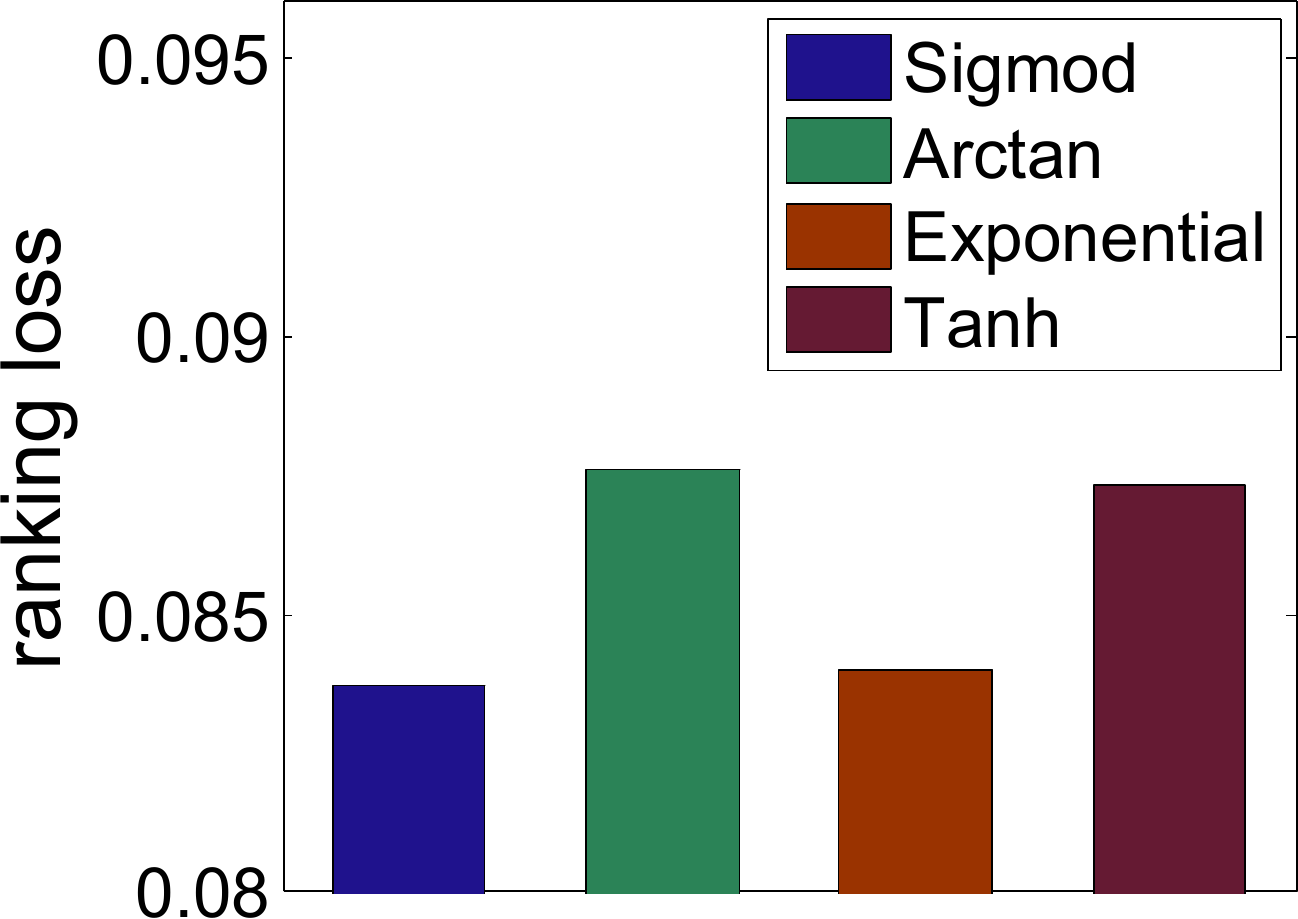}}
\subfigure{\includegraphics[width=0.19\linewidth]{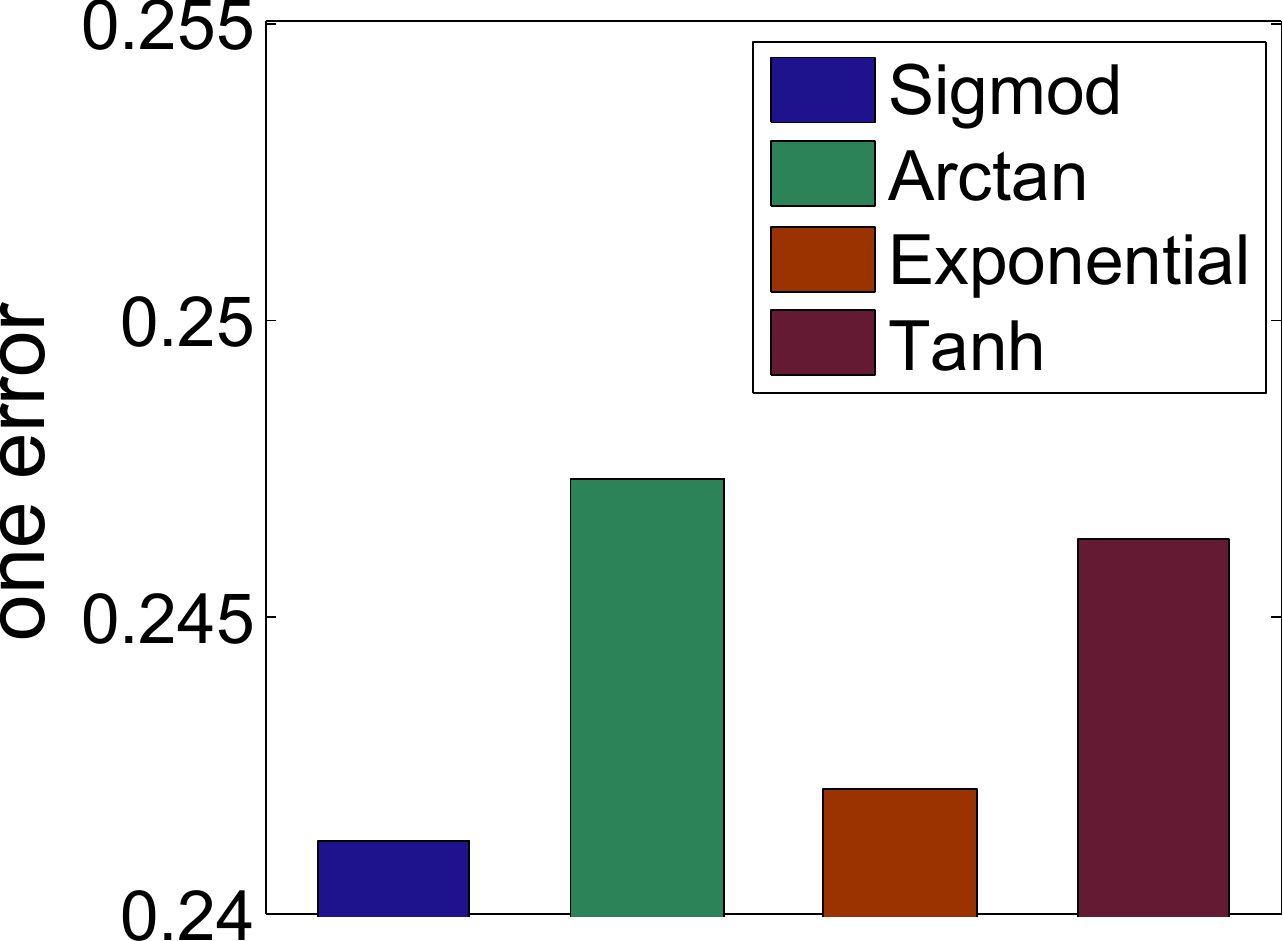}}
\subfigure{\includegraphics[width=0.19\linewidth]{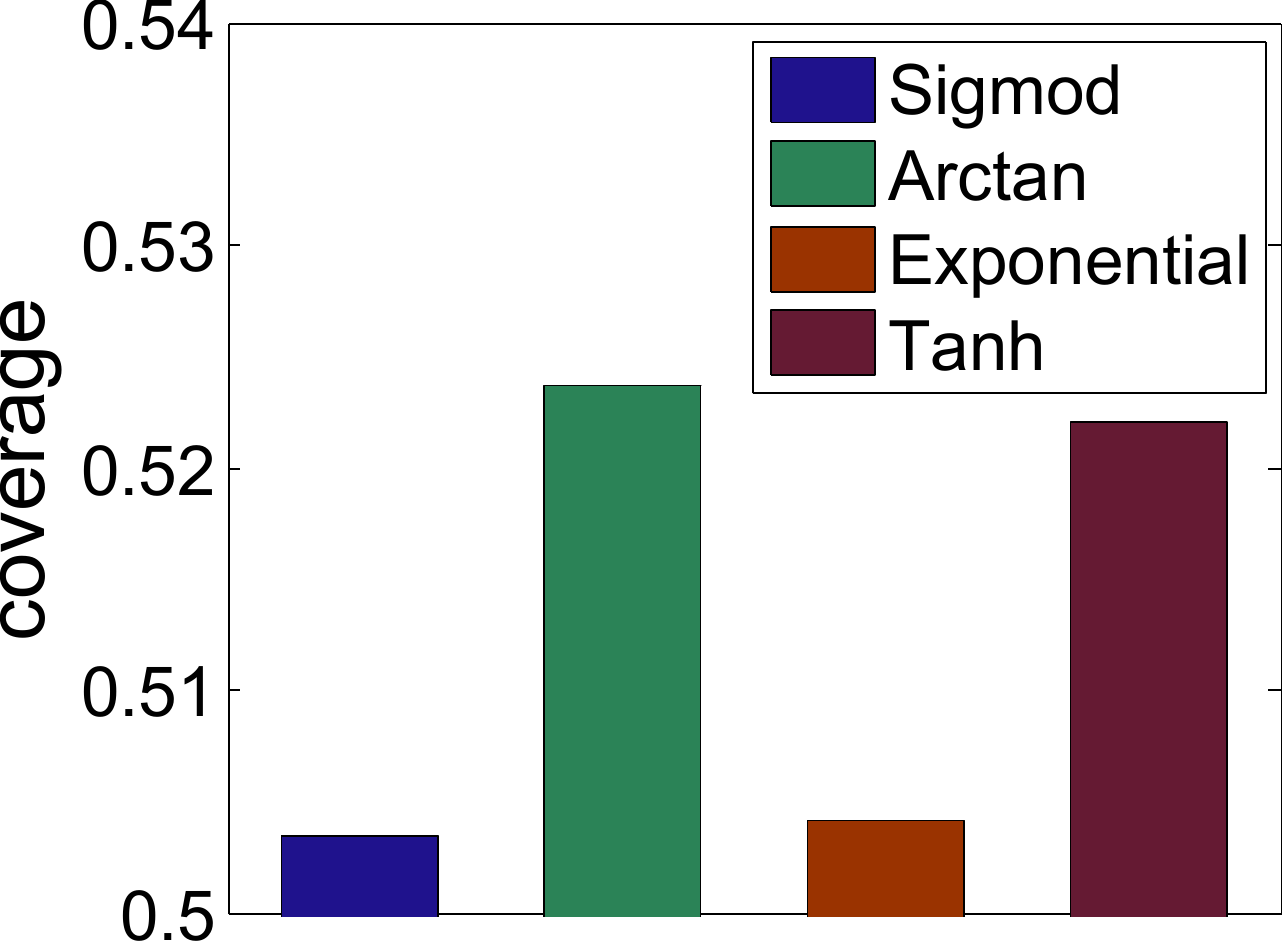}}
\subfigure{\includegraphics[width=0.19\linewidth]{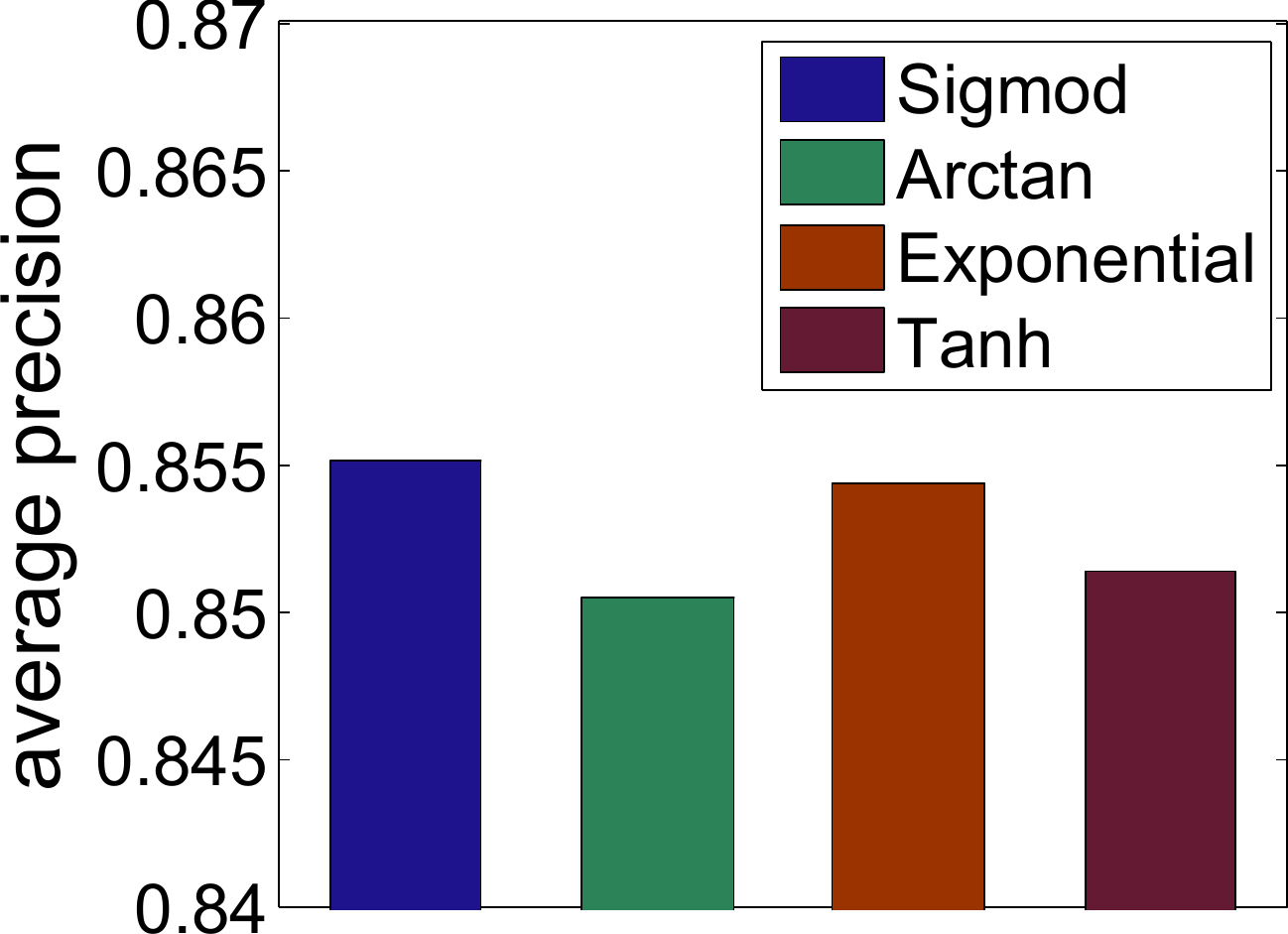}}
\subfigure{\includegraphics[width=0.19\linewidth]{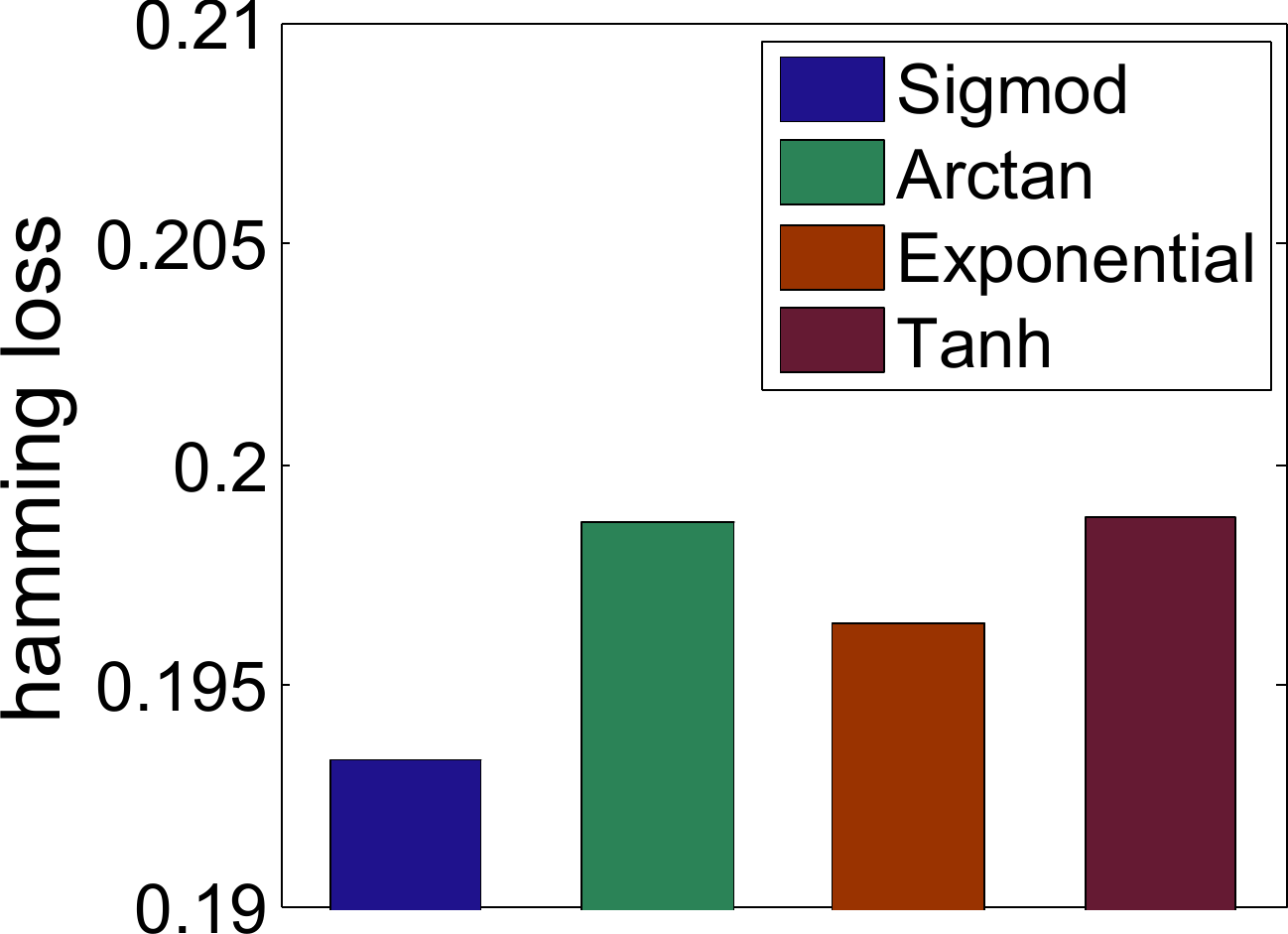}}
\subfigure{\includegraphics[width=0.19\linewidth]{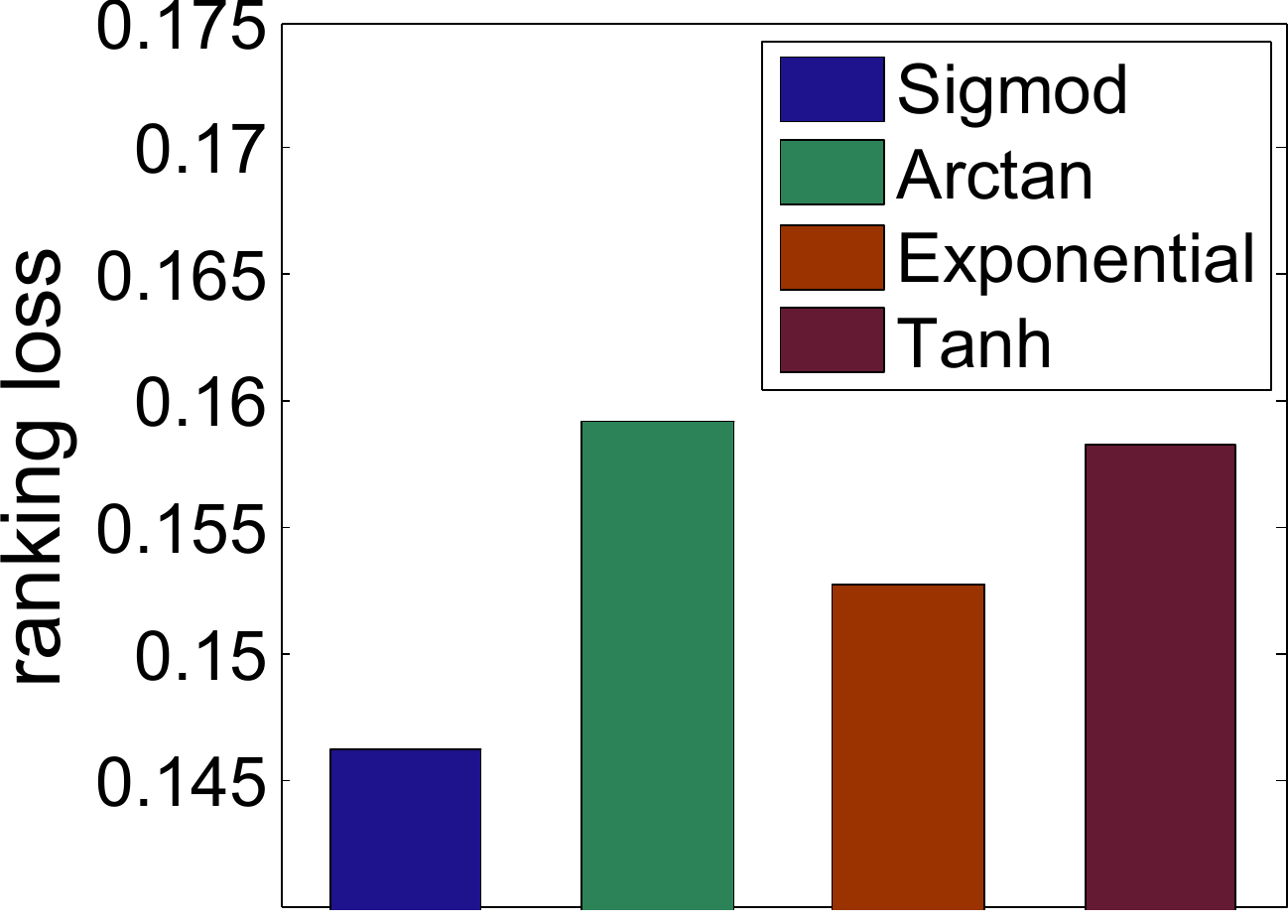}}
\subfigure{\includegraphics[width=0.19\linewidth]{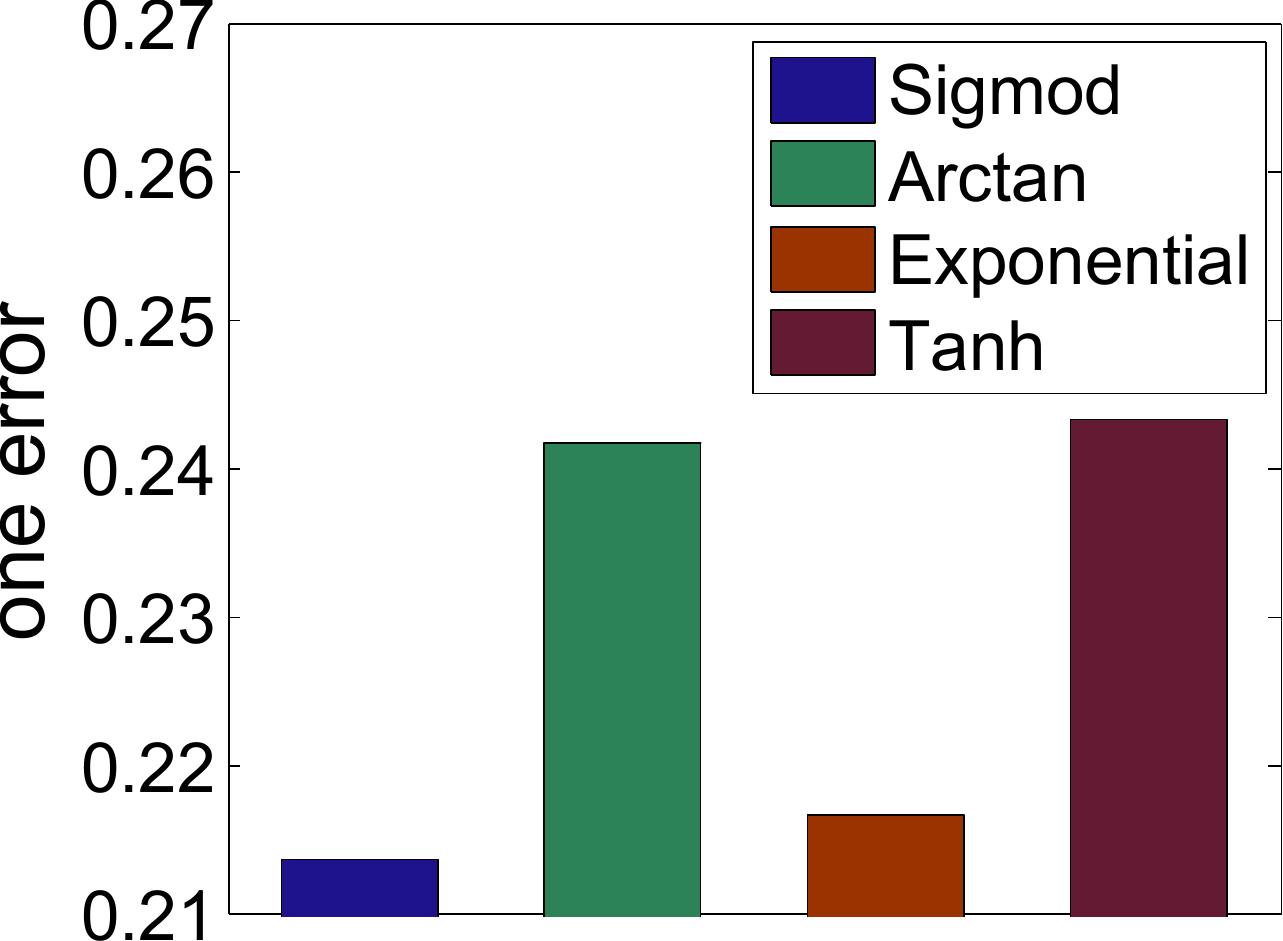}}
\subfigure{\includegraphics[width=0.19\linewidth]{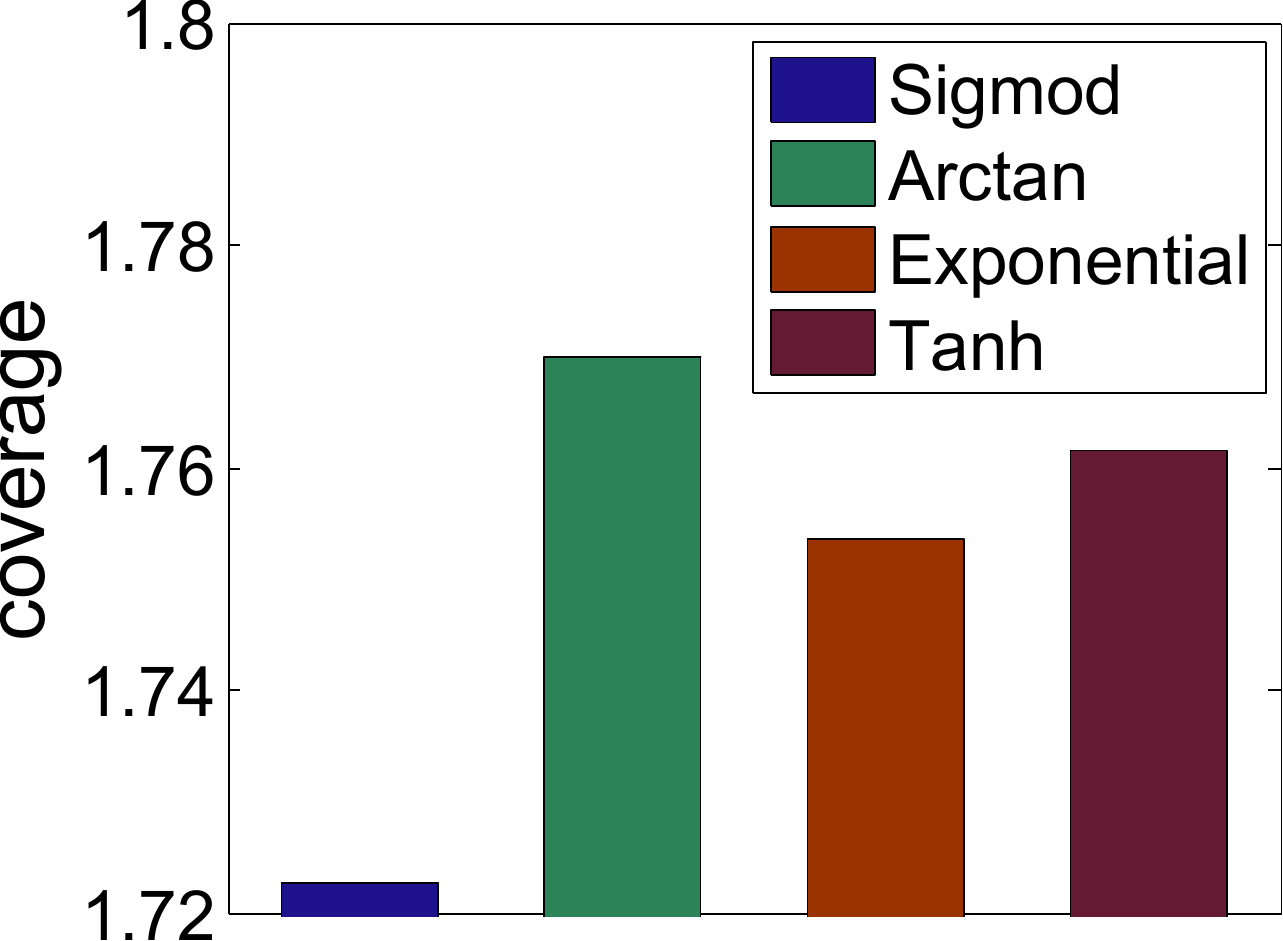}}
\subfigure{\includegraphics[width=0.19\linewidth]{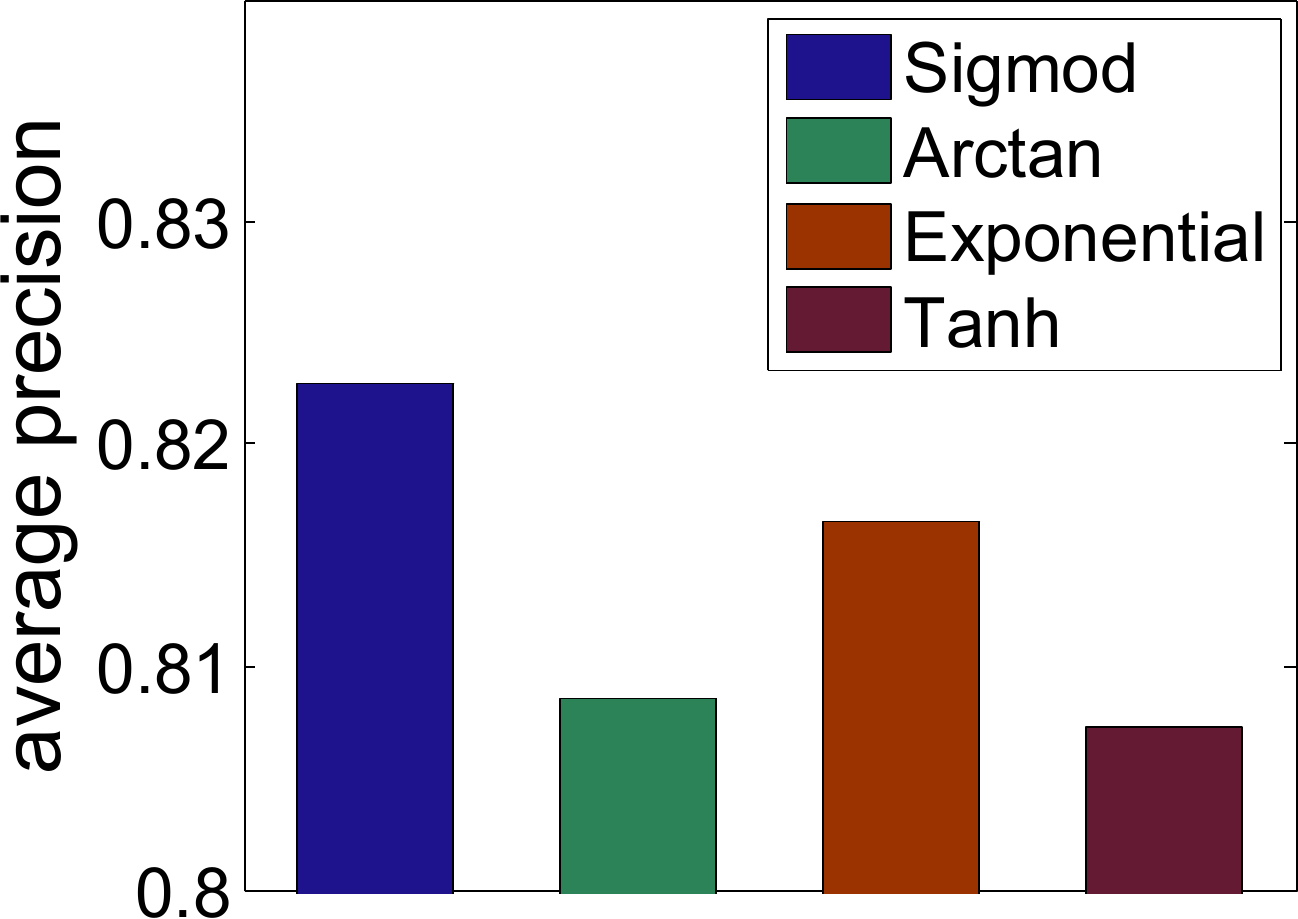}}
\caption{The performance of different self-paced functions on the \emph{flags}, \emph{scene}, and \emph{emotions} datasets. The first five figures are the results on the \emph{flags} dataset. The next five ones are the results on the  \emph{scene} dataset. The last five ones are the results on the \emph{emotions} dataset. }
\label{selfpacedfunction}
\end{figure*}
\begin{table*}[tb!]
\begin{center}
\caption{Results (mean $\pm$ std.) on the \emph{emotions} dataset. Boldface in the table denotes the best performance.}
\label{emotion}
\centering
\begin{tabular}{|c|c|c|c|c|c|c|c|}
\hline
 Dataset &   criteria          &   BSVM    &   ML-kNN   &  TRAM & RankSVM & ML-LOC & MLSPL \\
\hline
\multirow{10}{*}{\emph{emotions}}  & \multirow{2}{*}{hamming loss $\downarrow$}   &0.1937  & 0.2039   &   0.2210 &  0.3286   &   0.2047  & \textbf{0.1933} \\
                                  &     &$\pm$0.0086  &     $\pm$0.0095  &   $\pm$0.0093 &  $\pm$0.0123   &$\pm$0.0083 & $\pm$\textbf{0.0062}\\
 \cline{2-8} & \multirow{2}{*}{ranking loss $\downarrow$}            &0.1563  &    0.1714  &    0.1617  &  0.4043   & 0.1739  & \textbf{0.1462}\\
 &     &$\pm$0.0040  &  $\pm$0.0093  & $\pm$0.0075  & $\pm$0.0123   & $\pm$0.0083 & $\pm$\textbf{0.0001}\\
 \cline{2-8} & \multirow{2}{*}{one error $\downarrow$}            &0.2529 &  0.2788  &   0.2758 &  0.5549  &  0.2667  & \textbf{0.2138} \\
 &     &$\pm$0.0189  &  $\pm$0.0145  &   $\pm$0.0227  &  $\pm$0.0234   & $\pm$0.0173 & $\pm$\textbf{0.0017}\\
  \cline{2-8} & \multirow{2}{*}{coverage $\downarrow$}           &1.7882  &  1.8657  &  1.8024 &   3.0896 &  1.8879   & \textbf{1.7226} \\
  &     &$\pm$0.0306  &  $\pm$0.0615 &   $\pm$0.0524  &  $\pm$0.0917   & $\pm$0.0586 & $\pm$\textbf{0.0247}\\
   \cline{2-8} & \multirow{2}{*}{average precision $\uparrow$}           &0.8091  &  0.7910  &  0.7986 &   0.5789  &  0.7950  & \textbf{0.8228}\\
   &     &$\pm$0.0074  &    $\pm$0.0114  &   $\pm$0.0109  &  $\pm$0.0092   & $\pm$0.0084 & $\pm$\textbf{0.0005}\\
\hline
\end{tabular}
\end{center}
\end{table*}
\subsection{Studying the Performance of Self-Paced Functions}
In this section, we study the performance of different self-paced functions on the three datasets. The results are shown in Figure \ref{selfpacedfunction}. First of all, we can see that our method with different self-paced functions can achieve good performance. This shows that the self-paced functions we provide in the Sect. 2 are effective for multi-label learning.
In addition, we observe that on the \emph{flags} dataset, exponential function has the best performance in terms of all the five criteria except the Hamming distance, while on the \emph{emotions} dataset, sigmod function outperforms all the other functions. Tanh function performs similar to Arctan in terms of all the five criteria on the three datasets.
These points indicate that different scenarios indeed need different self-paced learning schemes. Therefore, it is necessary to develop more self-paced functions for multi-label learning.
\section{Conclusion}
We proposed a novel multi-label learning algorithm, namely MLSPL. By introducing a self-paced regularizer, MLSPL can learn labels according to the order of labels and instances from easy to hard. Considering that real-world scenarios usually needs different learning schemes,  we propose a general way to find the desired self-paced regularizer.
Experiments on benchmark datasets have demonstrated the effectiveness of SPMTL, compared to the state-of-the-art methods.

\bibliographystyle{IEEEtran}
\bibliography{IEEEabrv,ijcai16}

\begin{thebibliography}{10}
\providecommand{\url}[1]{#1}
\csname url@samestyle\endcsname
\providecommand{\newblock}{\relax}
\providecommand{\bibinfo}[2]{#2}
\providecommand{\BIBentrySTDinterwordspacing}{\spaceskip=0pt\relax}
\providecommand{\BIBentryALTinterwordstretchfactor}{4}
\providecommand{\BIBentryALTinterwordspacing}{\spaceskip=\fontdimen2\font plus
\BIBentryALTinterwordstretchfactor\fontdimen3\font minus
  \fontdimen4\font\relax}
\providecommand{\BIBforeignlanguage}[2]{{%
\expandafter\ifx\csname l@#1\endcsname\relax
\typeout{** WARNING: IEEEtran.bst: No hyphenation pattern has been}%
\typeout{** loaded for the language `#1'. Using the pattern for}%
\typeout{** the default language instead.}%
\else
\language=\csname l@#1\endcsname
\fi
#2}}
\providecommand{\BIBdecl}{\relax}
\BIBdecl

\bibitem{huang2012multi}
S.-J. Huang and Z.-H. Zhou, ``Multi-label learning by exploiting label
  correlations locally.'' in \emph{AAAI}, 2012.

\bibitem{zhang2014review}
M.-L. Zhang and Z.-H. Zhou, ``A review on multi-label learning algorithms,''
  \emph{IEEE Trans. on Knowledge and Data Engineering}, vol.~26, no.~8, pp.
  1819--1837, 2014.

\bibitem{nguyen2013multi}
C.-T. Nguyen, D.-C. Zhan, and Z.-H. Zhou, ``Multi-modal image annotation with
  multi-instance multi-label lda,'' in \emph{IJCAI}, 2013.

\bibitem{wang2008transductive}
J.~Wang, Y.~Zhao, X.~Wu, and X.-S. Hua, ``Transductive multi-label learning for
  video concept detection,'' in \emph{MIR}, 2008.

\bibitem{ji2008extracting}
S.~Ji, L.~Tang, S.~Yu, and J.~Ye, ``Extracting shared subspace for multi-label
  classification,'' in \emph{KDD}, 2008.

\bibitem{bucak2010multi}
S.~Bucak, R.~Jin, and A.~K. Jain, ``Multi-label multiple kernel learning by
  stochastic approximation: Application to visual object recognition,'' in
  \emph{NIPS}, 2010.

\bibitem{elisseeff2001kernel}
A.~Elisseeff and J.~Weston, ``A kernel method for multi-labelled
  classification,'' in \emph{NIPS}, 2001.

\bibitem{cai2004hierarchical}
L.~Cai and T.~Hofmann, ``Hierarchical document categorization with support
  vector machines,'' in \emph{CIKM}, 2004.

\bibitem{rousu2006learning}
J.~Rousu, C.~Saunders, S.~Szedmak, and J.~Shawe-Taylor, ``Learning hierarchical
  multi-category text classification models,'' \emph{JMLR}, vol.~7, pp.
  1601--1626, 2006.

\bibitem{cesa2006hierarchical}
N.~Cesa-Bianchi, C.~Gentile, and L.~Zaniboni, ``Hierarchical classification:
  combining bayes with svm,'' in \emph{ICML}, 2006.

\bibitem{hariharan2010large}
B.~Hariharan, L.~Zelnik-Manor, M.~Varma, and S.~Vishwanathan, ``Large scale
  max-margin multi-label classification with priors,'' in \emph{ICML}, 2010.

\bibitem{boutell2004learning}
M.~R. Boutell, J.~Luo, X.~Shen, and C.~M. Brown, ``Learning multi-label scene
  classification,'' \emph{Pattern recognition}, vol.~37, pp. 1757--1771, 2004.

\bibitem{zhu2005multi}
S.~Zhu, X.~Ji, W.~Xu, and Y.~Gong, ``Multi-labelled classification using
  maximum entropy method,'' in \emph{SIGIR}, 2005.

\bibitem{yan2007model}
R.~Yan, J.~Tesic, and J.~R. Smith, ``Model-shared subspace boosting for
  multi-label classification,'' in \emph{KDD}, 2007.

\bibitem{qi2007correlative}
G.-J. Qi, X.-S. Hua, Y.~Rui, J.~Tang, T.~Mei, and H.-J. Zhang, ``Correlative
  multi-label video annotation,'' in \emph{MM}, 2007.

\bibitem{zhang2010multi}
M.-L. Zhang and K.~Zhang, ``Multi-label learning by exploiting label
  dependency,'' in \emph{KDD}, 2010.

\bibitem{yang2013multi}
S.-J. Yang, Y.~Jiang, and Z.-H. Zhou, ``Multi-instance multi-label learning
  with weak label,'' in \emph{IJCAI}, 2013.

\bibitem{li2013active}
X.~Li and Y.~Guo, ``Active learning with multi-label svm classification,'' in
  \emph{IJCAI}, 2013.

\bibitem{ji2009multi}
S.~Ji, L.~Sun, R.~Jin, and J.~Ye, ``Multi-label multiple kernel learning,'' in
  \emph{NIPS}, 2009.

\bibitem{kumar2010self}
M.~P. Kumar, B.~Packer, and D.~Koller, ``Self-paced learning for latent
  variable models,'' in \emph{NIPS}, 2010.

\bibitem{jiang2014self}
L.~Jiang, D.~Meng, S.-I. Yu, Z.~Lan, S.~Shan, and A.~Hauptmann, ``Self-paced
  learning with diversity,'' in \emph{NIPS}, 2014.

\bibitem{zhao2015self}
Q.~Zhao, D.~Meng, L.~Jiang, Q.~Xie, Z.~Xu, and A.~G. Hauptmann, ``Self-paced
  learning for matrix factorization,'' in \emph{AAAI}, 2015.

\bibitem{xu2015multi}
C.~Xu, D.~Tao, and C.~Xu, ``Multi-view self-paced learning for clustering,'' in
  \emph{IJCAI}, 2015.

\bibitem{zhang2015self}
D.~Zhang, D.~Meng, C.~Li, L.~Jiang, Q.~Zhao, and J.~Han, ``A self-paced
  multiple-instance learning framework for co-saliency detection,'' in
  \emph{ICCV}, 2015.

\bibitem{jiang2015self}
L.~Jiang, D.~Meng, Q.~Zhao, S.~Shan, and A.~G. Hauptmann, ``Self-paced
  curriculum learning,'' in \emph{AAAI}, 2015.

\bibitem{chang2011libsvm}
C.-C. Chang and C.-J. Lin, ``Libsvm: A library for support vector machines,''
  \emph{TIST}, vol.~2, no.~3, p.~27, 2011.

\bibitem{zhang2007ml}
M.-L. Zhang and Z.-H. Zhou, ``Ml-knn: A lazy learning approach to multi-label
  learning,'' \emph{Pattern recognition}, vol.~40, no.~7, pp. 2038--2048, 2007.

\bibitem{kong2013transductive}
X.~Kong, M.~K. Ng, and Z.-H. Zhou, ``Transductive multilabel learning via label
  set propagation,'' \emph{TKDE}, vol.~25, no.~3, pp. 704--719, 2013.

\bibitem{schapire2000boostexter}
R.~E. Schapire and Y.~Singer, ``Boostexter: A boosting-based system for text
  categorization,'' \emph{Machine learning}, vol.~39, no.~2, pp. 135--168,
  2000.

\bibitem{zhou2012multi}
Z.-H. Zhou, M.-L. Zhang, S.-J. Huang, and Y.-F. Li, ``Multi-instance
  multi-label learning,'' \emph{Artificial Intelligence}, vol. 176, no.~1, pp.
  2291--2320, 2012.

\end{thebibliography}

\end{document}